\documentclass[runningheads]{llncs}

\usepackage[T1]{fontenc}
\def\doi#1{\href{https://doi.org/\detokenize{#1}}{\url{https://doi.org/\detokenize{#1}}}}
\usepackage{graphicx}
\usepackage{multicol}
\usepackage{color}
\usepackage{hyperref}
\definecolor{linkcolor}{rgb}{0.45,0.05,0.05}
\definecolor{citecolor}{rgb}{0.05,0.45,0.45}
\definecolor{urlcolor}{rgb}{0.05,0.05,0.45}
\hypersetup{
  linkcolor  = linkcolor,
  citecolor  = citecolor,
  urlcolor   = urlcolor,
  colorlinks = true,
  bookmarks = true,
}

\usepackage{epsfig} 
\usepackage[mathscr]{euscript}
\usepackage[ruled,vlined,noend,linesnumbered]{algorithm2e}
\usepackage{caption}
\usepackage{subcaption}
\usepackage{amsmath} 
\let\proof\relax\let\endproof\relax
\usepackage{amsthm}
\usepackage{amssymb}  
\usepackage{tikz}  
\usetikzlibrary{arrows,automata,positioning}
\usepackage{tabularx}
\usepackage{xspace}
\usepackage{wrapfig}
\usepackage{thmtools} 
\usepackage{thm-restate}
\usepackage[draft]{todonotes}
\usepackage{siunitx}

\providecommand{\gobble}[1]{}
\providecommand{\gobblexor}[2]{#2}

\usepackage{listings}
\definecolor{grey}{rgb}{0.5,0.5,0.5}
\definecolor{darkgrey}{rgb}{0.15,0.15,0.15}
\definecolor{darkblue}{rgb}{0.05,0.05,0.5}
\definecolor{darkgreen}{rgb}{0.05,0.5,0.05}
\definecolor{darkestgreen}{rgb}{0.5,0.0,0.5}
\definecolor{darkorange}{rgb}{0.5,0.25,0.00}

\providecommand{\defemp}[1]{\emph{#1}} 
\newcounter{tecounter}
\newenvironment{tightenumerate}
{
    \begin{list}{
    \arabic{tecounter}\addtocounter{tecounter}{1})}{%
    \setcounter{tecounter}{1}
        \setlength{\leftmargin}{08pt}
        \setlength{\topsep}{1pt}
        \setlength{\partopsep}{0pt}
        \setlength{\itemsep}{2pt}
        \setlength\labelwidth{0pt}}
        \ignorespaces}
{\unskip\end{list}}

\providecommand{\reachedvf}[3]{\ensuremath{\mathcal{V}_{#2}(#1, #3)}}
\providecommand{\reachedf}[2]{\ensuremath{\mathcal{V}_{#1}(#2)}}
\providecommand{\reachedv}[2]{\ensuremath{\mathcal{V}(#1, #2)}}
\providecommand{\reachedc}[2]{\ensuremath{\mathcal{C}(#1, #2)}}

\providecommand{\reaching}[2]{\ensuremath{\mathcal{S}^{#1}_{#2}}}

\providecommand{\Language}[1]{\ensuremath{\mathcal{L}(#1)}}
\providecommand{\extension}[2]{\ensuremath{\mathcal{L}_{#1}(#2)}}
\providecommand{\ALanguage}[1]{\ensuremath{\mathcal{L}^{A}(#1)}}
\newcommand*{\probleminternal}[4]{
    {\small
	\par
	\medskip
	\noindent\fbox{\parbox{0.98\columnwidth}{
		\textbf{#4:} {#1} \\[0.05in]
		\renewcommand{\tabcolsep}{2pt}
		\begin{tabularx}{0.9\linewidth}{rX}
			\qquad\emph{Input:} & #2 \\
			\qquad\emph{Output:} & #3
		\end{tabularx}
	}}}
	\par
	\medskip
	\par
}

\newcommand*{\ourproblem}[3]{\probleminternal{#1}{#2}{#3}{\quad~ Problem}}

\newcommand{\KleeneStr}[1]{\ensuremath{{#1}^{\ast}}}

\newcommand{\powSet}[1]{\raisebox{.15\baselineskip}{\large\ensuremath{\wp}}({#1})}

\newcommand{\vso}{vertex single-output\xspace}
\newcommand{\vmo}{vertex multi-output\xspace}
\newcommand{\sso}{string single-output\xspace}
\newcommand{\smo}{string multi-output\xspace}

\newcommand{\Sso}{String single-output\xspace}

\newenvironment{sproof}{%
  \proof}{\endproof}

\newcommand{\veryshortarrow}[1][4pt]{\mathrel{%
   \hbox{\rule[\dimexpr\fontdimen22\textfont2-.2pt\relax]{#1}{.4pt}}%
   \mkern-4mu\hbox{\usefont{U}{lasy}{m}{n}\symbol{41}}}}

\makeatletter

\setbox0\hbox{$\xdef\scriptratio{\strip@pt\dimexpr
    \numexpr(\sf@size*65536)/\f@size sp}$}

\newcommand{\scriptveryshortarrow}[1][3pt]{{%
    \hbox{\rule[\scriptratio\dimexpr\fontdimen22\textfont2-.2pt\relax]
               {\scriptratio\dimexpr#1\relax}{\scriptratio\dimexpr.4pt\relax}}%
   \mkern-4mu\hbox{\let\f@size\sf@size\usefont{U}{lasy}{m}{n}\symbol{41}}}}

\makeatother

\newcommand{\pto}{\ensuremath{\veryshortarrow\,}}
\newcommand{\fm}{{\sc fm}\xspace}
\newcommand{\dfm}{{\sc fm(df\pto \!df)}\xspace}
\newcommand{\sfm}{{\sc fm(df\pto \!sso)}\xspace}
\newcommand{\mfm}{{\sc fm(df\pto \!smo)}\xspace}
\newcommand{\unitdfm}{{\sc fm(df${}_{\#1}$\pto \!df)}\xspace}
\newcommand{\unitsfm}{{\sc fm(df${}_{\#1}$\pto \!sso)}\xspace}
\newcommand{\unitmfm}{{\sc fm(df${}_{\#1}$\pto \!smo)}\xspace}
\newcommand{\gsfm}{{\sc fm(sso\pto \!sso)}\xspace}
\newcommand{\gmfm}{{\sc fm(smo\pto \!smo)}\xspace}
\newcommand{\titledfm}{{\sc \textbf{fm(df{\pto}\!df)}}\xspace}
\newcommand{\titlesfm}{{\sc \textbf{fm(df{\pto}\!sso)}}\xspace}
\newcommand{\titlemfm}{{\sc \textbf{fm(df{\pto}\!smo)}}\xspace}
\newcommand{\titleunitdfm}{{\sc \textbf{fm(df${}_{\#1}$\pto \!df)}}\xspace}
\newcommand{\titleunitsfm}{{\sc \textbf{fm(df${}_{\#1}$\pto \!sso)}}\xspace}
\newcommand{\titleunitmfm}{{\sc \textbf{fm(df${}_{\#1}$\pto \!smo)}}\xspace}

\newcommand{\gsfmOC}{{\sc sso\pto \!sso}\xspace}
\newcommand{\gmfmOC}{{\sc smo\pto \!smo}\xspace}
\newcommand{\titlegsfmOC}{{\sc \textbf{sso}\pto \!\textbf{sso}}\xspace}
\newcommand{\titlegmfmOC}{{\sc \textbf{smo}\pto \!\textbf{smo}}\xspace}

\providecommand{\ptext}{{\footnotesize\textsf{P}}}
\providecommand{\nptext}{{\footnotesize\textsf{NP}}}
\providecommand{\pspacetext}{{\footnotesize\textsf{PSPACE}}}

\providecommand{\p}{\ptext\xspace}

\providecommand{\pspace}{\pspacetext\xspace}

\providecommand{\npcomplete}{{\nptext -complete}\xspace}
\providecommand{\pspacehard}{{\pspacetext -hard}\xspace}
\providecommand{\pspacecomplete}{{\pspacetext -complete}\xspace}

\providecommand{\struct}[1]{\mathscr{#1}\xspace}

\newcommand{\set}[1]{#1}

\newcommand{\Naturals}{\set{\mathbb{N}}}

\newcommand{\PositiveNaturals}{\Naturals^{> 0}}

\DeclareMathOperator*{\Motimes}{\text{\raisebox{0.25ex}{\scalebox{0.75}{$\bigotimes$}}}}
\newcommand{\pgproduct}{\Motimes}

\newcommand\superrestr[3]{{
  \left.\kern-\nulldelimiterspace 
  #1 
  \vphantom{\big|} 
  \right|_{#2}^{#3} 
  }}

\newcommand\restr[2]{{
  \left.\kern-\nulldelimiterspace 
  #1 
  \vphantom{\big|} 
  \right|_{#2} 
  }}

\newcommand{\citep}[1]{\cite{#1}}
\newcommand{\citet}[1]{\cite{#1}}

\renewcommand{\emptyset}{\varnothing}

\makeatletter
\newcommand*\bigcdot{\mathpalette\bigcdot@{.7}}
\newcommand*\bigcdot@[2]{\mathbin{\vcenter{\hbox{\scalebox{#2}{$\m@th#1\bullet$}}}}}
\makeatother

\begin{document}
%
\title{Nondeterminism subject to output commitment in combinatorial filters}

\author{Yulin Zhang \inst{1} \and Dylan A. Shell\inst{2}}

\authorrunning{Y. Zhang, D. A. Shell}

%
%
%
\institute{The University of Texas at Austin, Austin TX 78712, USA  \and Texas
A\&M University, College Station TX 77843, USA\\
\email{yulin@cs.utexas.edu}\qquad \email{dshell@tamu.edu}
}
\maketitle              
\vspace*{-2.0ex}
\begin{abstract}

We study a class of filters\,---discrete finite-state transition systems
employed as incremental stream transducers---\,that have application to
robotics: e.g., to model combinatorial estimators and also as concise encodings of feedback plans\slash\hspace{0pt}policies. 
The present paper examines their minimization problem under some new assumptions.
Compared to strictly deterministic filters, allowing nondeterminism supplies
opportunities for compression via re-use of states.  But this paper suggests
that the classic automata-theoretic concept of nondeterminism, though it affords
said opportunities for reduction in state complexity, is problematic in many
robotics settings.  Instead, we argue for a new constrained type of
nondeterminism that preserves input--output behavior for circumstances when, as
for robots, causation forbids `rewinding' of the world.  We identify  problem
instances where compression under this constrained form of nondeterminism
results in improvements over all deterministic filters.  
%
In this new setting, we examine computational complexity questions for the problem of
reducing the state complexity of some given input filter.
A hardness result for general deterministic input filters is presented,
as well as for checking specific, narrower requirements, and some special cases.
These results show that this class of nondeterminism gives problems of the
same complexity
class as classical nondeterminism, and the narrower questions help give a
more nuanced understanding of the source of this complexity. 
%
\end{abstract}


\vspace*{-3.8ex}
\section{Introduction}
\vspace*{-0.8ex}

Going right back to Rabin and Scott's seminal paper~\cite{rabin59finite},
classic automata theory has considered a particular type of nondeterminism that
declares a string accepted if \textsl{some} tracing of it reaches a final state.
Under this definition, as that paper first established, such nondeterministic
finite automata express the same languages as their deterministic brethren,
\emph{viz.\ }%
precisely the regular ones.  But it is well known that nondeterminism can,
nevertheless, confer practical benefits: nondeterminism permits expression
of some languages with far greater concision. 

These facts have bearing on a line of robotics research looking at finite-state
encodings of combinatorial estimators, discrete feedback plans and stateful
policies. Employing the terminology of Tovar
et~al.~\cite{tovar2014combinatorial}, we shall refer to discrete transition
systems that process a stream of observations as `filters'.  
In their paper, they demonstrate how ascertaining the amount of information
required to answer particular queries may involve surprising subtlety. Among
their illustrative examples of filters, their most elegant instances track all
that is needed via curiously few states.
More generally, as representations, discrete filters can help direct attention
to considerations of \emph{minimalism}~\cite{connell1990minimalist}, drawing into sharp focus questions of
necessity rather than of mere sufficiency. Beside elucidating interesting
structure\,---a consideration of obvious scientific value---\,they also have
practical application in building simpler, cheaper devices.

A series of papers has explored \gobble{the question of }filter compression: proposing
algorithms for reducing the number of states
needed~\cite{zhang20cover,rahmani2020integer,zhang19accelerating}, and
examining the hardness of achieving the minimal filter under differing
assumptions~\cite{o2017concise,saberifar2017combinatorial,zhang2021nondeterminism,saberifar17inconsequential}.
An initial suggestion that the problem might simply be identical to automata
minimization was shown to be false~\cite{o2017concise}. And thoughts that filter
minimization could be accomplished by quotienting under a bisimulation relation turns out also to be
false~\cite{rahmani2018relationship} and, indeed, no equivalence relation will
do~\cite{zhang20cover}. Even though automata theory suggests that
nondeterminism may afford opportunities for added compression, no
algorithm for compressing filters currently exploits nondeterminism.

Insofar as nondeterminism does figure in prior work on filters, there are two
forms. The first is \emph{tracing nondeterminism} wherein any vertex may have
multiple departing arcs that match some symbol being processed.  This type of
nondeterminism corresponds with the concept in classical automata, and the
informal intuitions that are usual there, apply here as well.  So, when there
are two edges that match, we might speak of `taking both'; or, upon reaching
such a juncture, we might pick one but later change our mind and rewind to
choose another.  
The imaginary processes that these two narrations provide as interpretations,
despite being distinct, agree in terms of the language they characterize.

In filters there is a second form, \emph{output nondeterminism}, where any vertex may bear multiple outputs.  
On arriving at a vertex with
several outputs, any of these may be selected. Both types of filter
nondeterminism were first explored together in~\cite{setlabelrss}, which
investigated their relationship under a model that examines how sensor imperfections
lead to loss of functionality.  Looking specifically at minimization of
filters, output nondeterminism is formulated and studied in our paper at the
previous WAFR~\cite{zhang20cover}.  Recently, \cite{zhang2021nondeterminism}
added tracing nondeterminism to the minimization picture, showing that tracing
nondeterminism does allow further compression. (In other words, the
`opportunities' mentioned in the earlier paragraph \textsl{do} exist in
filters.)

\gobblexor{Starting from the next section but one, the}{The} present paper highlights why tracing
nondeterminism may not always be appropriate in robotics applications. It then defines
a new class of nondeterminism on the basis of this observation.
This class represents a \emph{juste milieu}: permitting choices in tracing while still encoding deterministic input--output behavior.
We then see whether this class affords additional
compression over the strictly deterministic case (spoiler: it does). 
Thereafter, Sections~\ref{sec:generalcase}--\ref{sec:diff} examine the
relation of this class, in terms of hardness of minimization, to others and
draws connections with results from automata theory.

%

\section{Basic definitions}

Having reached the limits of informal talk, 
some definitions follow next.

\begin{definition}[procrustean filter~\cite{setlabelrss}]
A \defemp{procrustean filter}, \defemp{p-filter} or \defemp{filter} for short,
is a 6-tuple $\struct{F} = (V, V_0, Y, \tau, C, c)$ in which $V$ is a non-empty finite set of
states, $V_0$ is the set of initial states, $Y$ is the set of
observations, $\tau: V\times V\rightarrow \powSet{Y}$ is
the transition function, $C$ is the set of outputs, and $c: V\to
\powSet{C}\setminus\{\emptyset\}$ is the output function.  
(Here, $\powSet{A}$ denotes the powerset of set $A$.)
\end{definition}

For some $\struct{F}$, we will write $V(\struct{F})$, $V_0(\struct{F})$ and
$Y(\struct{F})$, for its sets of states, initial states, and observations,
respectively.
We'll present filters visually as graphs, with states as vertices and
transitions as directed edges with 
observations.  We shall assume that
$Y(\struct{F})$ is finite (it affords some simplicity and will suffice for our
needs here, though cf.~\cite{setlabelrss}). To be consistent with automata
theory, we also call $Y(\struct{F})$ the alphabet of $\struct{F}$.
The outputs will appear visually as colors at each vertex, going some way to explain the
names $C$ and $c(\cdot)$. We say that $\struct{F}$ is \defemp{\vso}, if every
vertex $v$ in $\struct{F}$ has $|c(v)|=1$. Otherwise, it is \defemp{\vmo},
a term rather less vague than the phrase `output nondeterminism' used in the preceding section.

For some filter $\struct{F}=(V, V_0, Y, \tau, C, c)$, an observation sequence
(or a string) $s=y_1y_2\dots y_n\in \KleeneStr{Y}$, and states $v, w \in V$, we
say that $w$ is \defemp{reached by} $s$ (or $s$ \defemp{reaches} $w$) when
traced from $v$, if some sequence of states $w_0,w_1, \dots, w_{n}$ in
$\struct{F}$, such that $w_0 = v$, $w_n = w$, and $\forall i\in \lbrace 1, 2, \dots,
n\rbrace, y_i\in \tau(w_{i-1}, w_i)$.
The set of all states reached by $s$ from a state $v$ in $\struct{F}$ will be 
denoted $\reachedvf{v}{\struct{F}}{s}$, and we will 
use $\reachedf{\struct{F}}{s}$ for 
all states reached by $s$ from any initial
state of the filter, i.e.,
$\reachedf{\struct{F}}{s} = \bigcup_{v_0 \in V_0} \reachedvf{v_0}{\struct{F}}{s}$.  If
$\reachedvf{v}{\struct{F}}{s}$ is empty, then we say that string $s$ \emph{crashes} in
$\struct{F}$ starting from $v$. Otherwise, we say that $s$ is an
\defemp{extension} of $v$ in $\struct{F}$. The set of all extensions of $v$ in $\struct{F}$ is
denoted as $\extension{\struct{F}}{v}=\{s\in \KleeneStr{Y}|
\reachedvf{v}{\struct{F}}{s}\neq\emptyset\}$. Specifically, the set of all strings
that are extensions of any initial state in $\struct{F}$ is called the
\defemp{interaction language} (or, briefly, just \defemp{language}) of
$\struct{F}$, and is written as $\Language{\struct{F}}=\cup_{v_0\in
V_0(\struct{F})}\extension{\struct{F}}{v_0}$.
Contrariwise, the set of strings reaching $w$ from some initial state in $\struct{F}$ is denoted as
$\reaching{\struct{F}}{w}=\{s\in \KleeneStr{Y}| w\in\reachedf{\struct{F}}{s}\}$.
Without loss of generality and to help dispose of turgid statements of conditions, we shall assume that every state in $\struct{F}$
can be reached by some string from an initial state. Otherwise, we may remove
this state from $\struct{F}$ with no impact on the language
of $\struct{F}$ or outputs for strings in the language. 

Although quite standard, the following is paramount, so justifies emphasis:
\begin{definition}[tracing-deterministic]
A filter $\struct{F}=(V, V_0, Y, \tau, C, c)$ is \defemp{tracing-deterministic} or
\defemp{state-determined}, if $|V_0|=1$, and for every $v_1, v_2, v_3\in V$ with
$v_2\neq v_3$, $\tau(v_1, v_2)\cap \tau(v_1, v_3)=\emptyset$. 
\end{definition}

We shall say that a filter which is not tracing-deter\-min\-istic is tracing-non\-deter\-ministic. 
Examples of tracing-deterministic and tracing-nondeterministic filters appear in 
Figure~\ref{fig:ex_sso_first} and~\ref{fig:ex_sso_input}.

Next, we define the crucial concept of functional substitutability,
giving conditions when one filter may serve as a replacement for another.

\begin{definition}[output simulating~\cite{o2017concise}]
\label{def:stdos}
Let $\struct{F}$ and $\struct{F}'$ be two filters, then $\struct{F}'$ \defemp{output simulates} $\struct{F}$ if
(1). $\Language{\struct{F}}\subseteq \Language{\struct{F}'}$ and
(2). $\reachedc{\struct{F}'}{s}\subseteq \reachedc{\struct{F}}{s}$.
\end{definition}
For convenience, we refer to property (1) and (2) as \emph{language inclusion} and
\emph{output compatibility}, respectively.  The intuition is that language
inclusion ensures that $\struct{F}'$ is able to process any input that
$\struct{F}$ can.  When the output that $\struct{F}'$ produces could be some
output produced by $\struct{F}$, then it is considered compatible.

The three vertex single-output filters in Figure~\ref{fig:ex_sso} provide a basic feel for the
concept; shortly we shall consider vertex multi-output instances too.

\smallskip

The core optimization question is one of reducing state complexity in a filter:

\ourproblem{\textbf{Filter Minimization (\fm)}}
{A filter $\struct{F}$.}
{A filter $\struct{F}^{\dagger}$ with fewest states, such that $\struct{F}^{\dagger}$ output
simulates~$\struct{F}$.
}

The decision version of \fm provides some $k \in \Naturals$ and asks if any
$\struct{F}^{\dagger}$ with no more than $k$ states can output simulate $\struct{F}$.
We shall consider variations on this problem under different constraints (to be detailed in Section~\ref{sec:summary_of_opts}).

\begin{figure}[h]
\vspace*{-2pt}
\begin{subfigure}{0.3\textwidth}
\hspace*{-4ex}
\scalebox{0.5}{
\newcommand\vs{0.75}
\newcommand\hs{1.0}
\newcommand\hsb{1.0}
\newcommand\hsa{1.8}
{
\begin{tikzpicture}[shorten >=1pt,node distance=1cm, on grid,auto]
\tikzset{every state/.style={semithick, minimum size=10pt, inner sep=8pt,circle}}
\tikzset{every path/.style={thick, font=\Large}}
\tikzset{initial distance=0.5cm}
\node[state, initial, fill=white] (q0)   {}; 
   \node[state] (p1) [above right=\vs cm and 1.8cm of q0, fill=blue!25] {}; 
   \node[state] (p2) [right=\hsa of p1, fill=green!20] {}; 
   \node[state] (r2) [right=\hsa of p2, fill=gray!20] {}; 
   \node[state] (p3) [below right=\vs cm and 1.8cm of q0, fill=pink!90] {}; 
   \node[state] (p4) [right=\hsa of p3, fill=orange!40] {}; 
   \node[state] (r3) [gray, densely dashed,right=\hsa of p4, fill=yellow!40] {}; 
   
    \path[->] 
    (q0) edge node [pos=0.5, sloped, above, inner sep=1.2ex] {$a,b$} (p1)
    (p1) edge node [pos=0.5, sloped, above, inner sep=1.2ex] {$x$} (p2)
    (p2) edge node [pos=0.5, sloped, above, inner sep=1.2ex] {$y$} (r2)
    (q0) edge node [pos=0.5, sloped, below, inner sep=1.2ex] {$a,c$} (p3)
    (p3) edge node [pos=0.5, sloped, below, inner sep=1.2ex] {$x$} (p4)
    ;

    \path[gray,densely dashed,->] 
    (p4) edge node [pos=0.5, sloped, below, inner sep=1.2ex] {$z$} (r3)
    ;

\end{tikzpicture}
}
}
\vspace*{-15pt}
\caption{
\label{fig:ex_sso_input}}
\end{subfigure}
\hfill
\begin{subfigure}{0.3\textwidth}
\hspace*{-4ex}
\scalebox{0.5}{
\newcommand\vs{0.75}
\newcommand\hs{1.0}
\newcommand\hsb{1.0}
\newcommand\hsa{1.8}
{
\begin{tikzpicture}[shorten >=1pt,node distance=1cm, on grid,auto]
\tikzset{every state/.style={semithick, minimum size=10pt, inner sep=8pt,circle}}
\tikzset{every path/.style={thick, font=\Large}}
\tikzset{initial distance=0.5cm}
\node[state, initial, fill=white] (q0)   {}; 
   \node[state] (p1) [above right=\vs cm and 1.8cm of q0, fill=blue!25] {}; 
   \node[state] (p2) [right=\hsa of p1, fill=green!20] {}; 
   \node[state] (r2) [right=\hsa of p2, fill=gray!20] {}; 
   \node[state] (p3) [below right=\vs cm and 1.8cm of q0, fill=pink!90] {}; 
   \node[state] (p4) [right=\hsa of p3, fill=orange!40] {}; 
   
    \path[->] 
    (q0) edge node [pos=0.5, sloped, above, inner sep=1.2ex] {$a,b$} (p1)
    (p1) edge node [pos=0.5, sloped, above, inner sep=1.2ex] {$x$} (p2)
    (p2) edge node [pos=0.5, sloped, above, inner sep=1.2ex] {$y$} (r2)
    (q0) edge node [pos=0.5, sloped, below, inner sep=1.2ex] {$c$} (p3)
    (p3) edge node [pos=0.5, sloped, below, inner sep=1.2ex] {$x$} (p4)
    ;

\end{tikzpicture}
}
}
\vspace*{-15pt}
\caption{
\label{fig:ex_sso_first}}
\end{subfigure}
\hfill
\begin{subfigure}{0.3\textwidth}
\hspace*{-5ex}
\scalebox{0.5}{
\newcommand\vs{0.75}
\newcommand\hs{1.0}
\newcommand\hsb{1.0}
\newcommand\hsa{1.8}
{
\begin{tikzpicture}[shorten >=1pt,node distance=1cm, on grid,auto]
\tikzset{every state/.style={semithick, minimum size=10pt, inner sep=8pt,circle}}
\tikzset{every path/.style={thick, font=\Large}}
\tikzset{initial distance=0.5cm}
\node[state, initial, fill=white] (q0)   {}; 
   \node[state] (p1) [above right=\vs cm and 1.8cm of q0, fill=blue!25] {}; 
   \node[state] (p2) [right=\hsa of p1, fill=green!20] {}; 
   \node[state] (r2) [right=\hsa of p2, fill=gray!20] {}; 
   \node[state] (p3) [below right=\vs cm and 1.8cm of q0, fill=pink!90] {}; 
   \node[state] (p4) [right=\hsa of p3, fill=orange!40] {}; 
   \node[state] (r3) [right=\hsa of p4, fill=yellow!40] {}; 
   
    \path[->] 
    (q0) edge node [pos=0.5, sloped, above, inner sep=1.2ex] {$a,b$} (p1)
    (p1) edge node [pos=0.5, sloped, above, inner sep=1.2ex] {$x$} (p2)
    (p2) edge node [pos=0.5, sloped, above, inner sep=1.2ex] {$y$} (r2)
    (q0) edge node [pos=0.5, sloped, below, inner sep=1.2ex] {$c$} (p3)
    (p3) edge node [pos=0.5, sloped, below, inner sep=1.2ex] {$x$} (p4)
    (p4) edge node [pos=0.5, sloped, below, inner sep=1.2ex] {$z$} (r3)
    (p2) edge node [pos=0.5, sloped, above, inner sep=1.2ex] {$z$} (r3)
    ;

\end{tikzpicture}
}
}
\vspace*{-15pt}
\caption{
\label{fig:ex_sso_second}}
\end{subfigure}
\vspace*{-8pt}
\caption{Three simple vertex single-output filters: 
(a) a tracing-nondeterministic filter; 
(b) a tracing-deterministic one that output simulates the filter in (a) when the dashed $z$-edge and yellow-state are absent;
and (c) another one that output simulates (a), now with the dashed edge\slash\hspace{0pt}state, or without, and also (b) as well.
\label{fig:ex_sso}}
\vspace*{-10pt}
\end{figure}
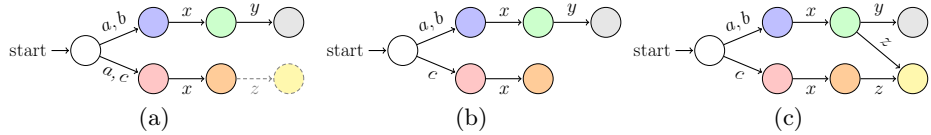



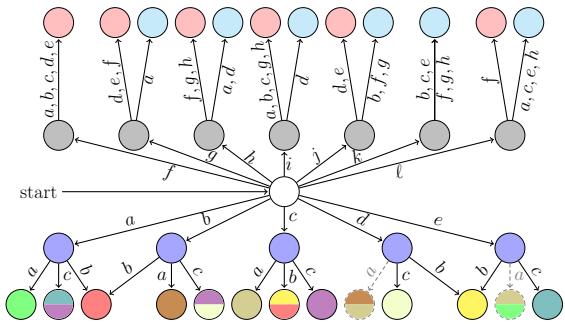
\begin{figure}[b]
\vspace*{-4pt}
\begin{minipage}{0.55\textwidth}
\vspace*{-6ex}
\hspace*{-4ex}
\scalebox{0.5}{
\newcommand\vs{1.5}
\newcommand\vS{3.0}
\newcommand\hsb{1.0}
\newcommand\hsa{2.0}
\newcommand\hsahalf{0.5}
\makeatletter
\tikzset{circle split part fill/.style  args={#1,#2}{%
 alias=tmp@name, 
  postaction={%
    insert path={
     \pgfextra{%
     \pgfpointdiff{\pgfpointanchor{\pgf@node@name}{center}}%
                  {\pgfpointanchor{\pgf@node@name}{east}}%
     \pgfmathsetmacro\insiderad{\pgf@x}
      \fill[#1] (\pgf@node@name.base) ([xshift=-\pgflinewidth]\pgf@node@name.east) arc
                          (0:180:\insiderad-\pgflinewidth)--cycle;
      \fill[#2] (\pgf@node@name.base) ([xshift=\pgflinewidth]\pgf@node@name.west)  arc
                           (180:360:\insiderad-\pgflinewidth)--cycle;            
         }}}}}  
 \makeatother

{
\begin{tikzpicture}[shorten >=1pt,node distance=1cm, on grid, auto] 
\tikzset{every state/.style={semithick, minimum size=10pt, inner sep=8pt,circle}}
\tikzset{every path/.style={thick, font=\Large}}
\tikzset{initial distance=5.5cm}
\node[state, initial, fill=white] (q0)   {}; 
   \node[state] (q1) [below left=\vs and 6cm of q0, fill=blue!35] {}; 
   \node[state] (q2) [right=3cm of q1,              fill=blue!35] {}; 
   \node[state] (q5) [right=3cm of q2,              fill=blue!35] {}; 
   \node[state] (q3) [right=3cm of q5,              fill=blue!35] {}; 
   \node[state] (q4) [right=3cm of q3,              fill=blue!35] {}; 
   \node[state] (q6) [below left=\vs cm and \hsb cm of q1, fill=green!50] {}; 
   \node[state] (q7) [below =\vs cm of q1, circle split part fill={teal!50,violet!50}] {}; 
   \node[state] (q8) [below right=\vs cm and \hsb of q1, fill=red!50] {}; 
   \node[state] (q9) [below =\vs cm of q2, fill=brown!90] {}; 
   \node[state] (q10) [below right=\vs cm and \hsb cm of q2, circle split part fill={violet!50,lime!20}] {}; 

   \node[state] (q12) [below =\vs cm  of q3, fill=lime!20] {}; 
   \node[state] (q13) [below left=\vs cm and \hsb cm of q4, fill=yellow!75] {}; 
   \node[state] (q15) [below right =\vs cm and \hsb cm of q4, fill=teal!50] {}; 
   \node[state] (q16) [below left=\vs cm and \hsb cm of q5, fill=olive!40] {}; 
   \node[state] (q17) [below =\vs cm of q5, circle split part fill={yellow!75,red!50}] {}; 
   \node[state] (q18) [below right=\vs cm and \hsb cm of q5, fill=violet!50] {}; 
   \node[state] (q11) [gray, densely dashed,below left =\vs cm and \hsb of q3, circle split part fill={brown!90,olive!40}] {}; 
   \node[state] (q14) [gray, densely dashed,below=\vs cm of q4, circle split part fill={olive!40,green!50}] {}; 

    \path[->] 
    (q0) edge node [pos=0.7, sloped, above] {$a$} (q1)
    (q0) edge node [pos=0.7, sloped, above] {$b$} (q2)
    (q0) edge node [pos=0.7, sloped, above] {$d$} (q3)
    (q0) edge node [pos=0.7, sloped, above] {$e$} (q4)
    (q0) edge node [pos=0.4, sloped, right, rotate=90] {$c$} (q5)

    (q1) edge node [pos=0.5, sloped, above] {$a$} (q6)
    (q1) edge node [pos=0.5, sloped, right, rotate=90] {$c$} (q7)
    (q1) edge node [pos=0.5, sloped, above] {$b$} (q8)
    (q2) edge node [pos=0.5, sloped, above] {$b$} (q8)
    (q2) edge node [pos=0.5, sloped, left, rotate=90] {$a$} (q9)
    (q2) edge node [pos=0.5, sloped, above] {$c$} (q10)
    (q3) edge node [pos=0.5, sloped, right, rotate=90] {$c$} (q12)
    (q3) edge node [pos=0.5, sloped, above] {$b$} (q13)
    (q4) edge node [pos=0.5, sloped, above] {$b$} (q13)
    (q4) edge node [pos=0.5, sloped, above] {$c$} (q15)
    (q5) edge node [pos=0.5, sloped, above] {$a$} (q16)
    (q5) edge node [pos=0.5, sloped, right, rotate=90] {$b$} (q17)
    (q5) edge node [pos=0.5, sloped, above] {$c$} (q18);

    \path[gray,densely dashed,->] 
    (q3) edge node [pos=0.5, sloped, above] {$a$} (q11)
    (q4) edge node [pos=0.5, sloped, right, rotate=90] {$a$} (q14);

   \node[state] (q1) [above left=\vs cm and 6cm of q0, fill=lightgray] {};
   \node[state] (q2) [right=\hsa cm of q1, fill=lightgray] {}; \node[state]
   (q3) [right=\hsa cm of q2, fill=lightgray] {}; \node[state] (q4) [right=\hsa
   cm of q3, fill=lightgray] {}; 
   \node[state] (q5) [right=\hsa cm of q4, fill=lightgray] {}; 
   \node[state] (q6) [right=\hsa cm of q5, fill=lightgray] {}; 
   \node[state] (q7) [right=\hsa cm of q6, fill=lightgray] {}; 

   \node[state] (plus1) [fill=pink, above =\vS cm of q1] {}; 
   \node[state] (plus2) [fill=pink, above left=\vS cm and \hsahalf cm of q2] {}; 
   \node[state] (plus3) [fill=pink, above left=\vS cm and \hsahalf cm of q3] {}; 
   \node[state] (plus4) [fill=pink, above left=\vS cm and \hsahalf cm of q4] {}; 
   \node[state] (plus5) [fill=pink, above left=\vS cm and \hsahalf cm of q5] {}; 
   \node[state] (plus7) [fill=pink, above left=\vS cm and \hsahalf cm of q7] {}; 

   \node[state] (minus2) [fill=cyan!20, above right=\vS cm and \hsahalf cm of q2] {}; 
   \node[state] (minus3) [fill=cyan!20, above right=\vS cm and \hsahalf cm of q3] {}; 
   \node[state] (minus4) [fill=cyan!20, above right=\vS cm and \hsahalf cm of q4] {}; 
   \node[state] (minus5) [fill=cyan!20, above right=\vS cm and \hsahalf cm of q5] {}; 
   \node[state] (minus6) [fill=cyan!20, above =\vS cm of q6] {}; 
   \node[state] (minus7) [fill=cyan!20, above right=\vS cm and \hsahalf cm of q7] {}; 

    \path[->] 
    (q0) edge node [pos=0.5, sloped, below, inner sep=0.2ex] {$f$} (q1)
    (q0) edge node [pos=0.5, sloped, above, inner sep=0.2ex] {$g$} (q2)
    (q0) edge node [pos=0.5, sloped, above, inner sep=0.2ex] {$h$} (q3)
    (q0) edge node [pos=0.5, sloped, right, inner sep=0.2ex, rotate=90] {$i$} (q4)
    (q0) edge node [pos=0.5, sloped, above, inner sep=0.2ex] {$j$} (q5)
    (q0) edge node [pos=0.5, sloped, above, inner sep=0.2ex] {$k$} (q6)
    (q0) edge node [pos=0.5, sloped, below, inner sep=0.2ex] {$\ell$} (q7)
    (q1) edge node [pos=0.5, sloped, above, inner sep=0.2ex] {$a,b,c,d,e$} (plus1)
    (q2) edge node [pos=0.5, sloped, above, inner sep=0.2ex, rotate=180] {$d,e,f$} (plus2)
    (q2) edge node [pos=0.5, sloped, below, inner sep=0.2ex] {$a$} (minus2)
    (q3) edge node [pos=0.5, sloped, above, inner sep=0.2ex, rotate=180] {$f,g,h$} (plus3)
    (q3) edge node [pos=0.5, sloped, below, inner sep=0.2ex] {$a,d$} (minus3)
    (q4) edge node [pos=0.5, sloped, above, inner sep=0.2ex, rotate=180] {$a,b,c,g,h$} (plus4)
    (q4) edge node [pos=0.5, sloped, below, inner sep=0.2ex] {$d$} (minus4)
    (q5) edge node [pos=0.5, sloped, above, inner sep=0.2ex, rotate=180] {$d,e$} (plus5)
    (q5) edge node [pos=0.5, sloped, below, inner sep=0.2ex] {$b,f,g$} (minus5)
    (q6) edge node [pos=0.5, sloped, above, inner sep=0.2ex] {$b,c,e$} (minus6)
    (q6) edge node [pos=0.5, sloped, below, inner sep=0.2ex] {$f,g,h$} (minus6)
    (q7) edge node [pos=0.5, sloped, above, inner sep=0.2ex, rotate=180] {$f$} (plus7)
    (q7) edge node [pos=0.5, sloped, below, inner sep=0.2ex] {$a,c,e,h$} (minus7);
    \end{tikzpicture}
}
}
\end{minipage}
\hfill
\begin{minipage}{0.35\textwidth}
\caption{A filter $\struct{F}_{\rm inp}$ with $Y(\struct{F}_{\rm inp})=\{a,b,c,\dots,k,\ell\}$ that is tracing-deterministic and vertex multi-output.
Tracing string `$ac$' may produce either 
\textcolor{violet!80}{\bf purple} or 
\textcolor{teal}{\bf teal} as an output.
(The dashed $a$-edges\slash\hspace{0pt}vertices should be ignored at first; they will be introduced when the discussion re-visits the example.)
\label{fig:finp}}
\end{minipage}
\vspace*{-8ex}
\end{figure}

\section{Some examples leading to our key definition}

Starting from an important problem instance, this section builds up to Definition~\ref{def:sso}, the
central definition of the paper.

\subsection{An interesting filter and two of its minimizers}

Consider Figure~\ref{fig:finp}. It shows an example of a tracing-nondeterministic and vertex multi-output filter, 
$\struct{F}_{\rm inp}$.  It has the form of a tree, and hence has a finite language.
When processing strings beginning with $\{a,b,c,d,e\}$, the tracing leads to what we'll term the lower-half; 
strings starting with $\{f,g,\dots,\ell\}$ lead to the upper-half.

Figure~\ref{fig:fdetout} gives a tracing-deterministic vertex single-output filter, $\struct{F}_{\rm det}$, that
output simulates $\struct{F}_{\rm inp}$.
In the upper-half, all the pink states have been merged together, and the same has occurred with the light-blue ones.
For the lower-half, single colors have been chosen for the multi-output vertices (picking purple and yellow)
and then possible mergers made. In fact, $\struct{F}_{\rm det}$ is a minimal tracing-deterministic filter
that output simulates $\struct{F}_{\rm inp}$: 

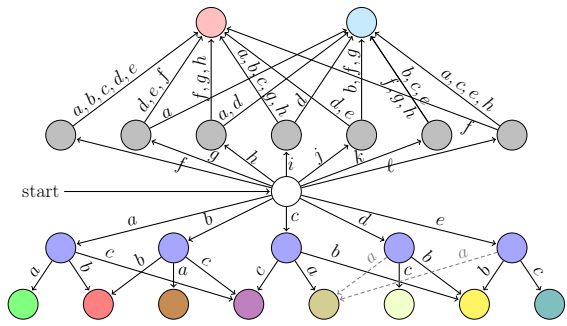
\begin{figure}[hbt!]
\begin{minipage}{0.55\textwidth}
\vspace*{-4ex}
\hspace*{-3ex}
\scalebox{0.5}{
\newcommand\vs{1.5}
\newcommand\vS{3.0}
\newcommand\hsb{1.0}
\newcommand\hsa{2.0}

{
\begin{tikzpicture}[shorten >=1pt,node distance=1cm, on grid, auto] 
\tikzset{every state/.style={semithick, minimum size=10pt, inner sep=8pt,circle}}
\tikzset{every path/.style={thick, font=\Large}}
\tikzset{initial distance=5.5cm}
   \node[state,initial, fill=white] (q0)   {}; 
   \node[state] (q1) [below left=\vs cm and 6cm of q0, fill=blue!35] {}; 
   \node[state] (q2) [right=3cm of q1, fill=blue!35] {}; 
   \node[state] (q5) [right=3cm of q2, fill=blue!35] {}; 
   \node[state] (q3) [right=3cm of q5, fill=blue!35] {}; 
   \node[state] (q4) [right=3cm of q3, fill=blue!35] {}; 
   \node[state] (q6) [below left=\vs cm and \hsb cm of q1, fill=green!50] {}; 
   \node[state] (q8) [below right=\vs cm and \hsb cm of q1, fill=red!50] {}; 
   \node[state] (q9) [below =\vs cm of q2, fill=brown!90] {}; 
   \node[state] (q12) [below =\vs cm of q3, fill=lime!20] {}; 
   \node[state] (q13) [below left=\vs cm and \hsb cm of q4, fill=yellow!75] {}; 
   \node[state] (q15) [below right=\vs cm and 1 cm of q4, fill=teal!50] {}; 
   \node[state] (q18) [below left=\vs cm and \hsb cm of q5, fill=violet!50] {}; 
   \node[state] (q16) [below right=\vs cm and \hsb cm of q5, fill=olive!40] {};

    \path[->] 
    (q0) edge node [pos=0.7, sloped, above] {$a$} (q1)
    (q0) edge node [pos=0.7, sloped, above] {$b$} (q2)
    (q0) edge node [pos=0.7, sloped, above] {$d$} (q3)
    (q0) edge node [pos=0.7, sloped, above] {$e$} (q4)
    (q0) edge node [pos=0.4, sloped, right, rotate=90] {$c$} (q5)

    (q1) edge node [pos=0.5, sloped, above] {$a$} (q6)
    (q1) edge node [pos=0.4, sloped, above] {$b$} (q8)
    (q1) edge node [pos=0.2, sloped, above] {$c$} (q18)
    (q2) edge node [pos=0.3, sloped, above] {$b$} (q8)
    (q2) edge node [pos=0.3, sloped, right, rotate=90] {$a$} (q9)
    (q2) edge node [pos=0.3, sloped, above] {$c$} (q18)
    (q3) edge node [pos=0.3, sloped, right, rotate=90] {$c$} (q12)
    (q3) edge node [pos=0.2, sloped, above] {$b$} (q13)
    (q4) edge node [pos=0.5, sloped, above] {$b$} (q13)
    (q4) edge node [pos=0.5, sloped, above] {$c$} (q15)
    (q5) edge node [pos=0.5, sloped, above] {$a$} (q16)
    (q5) edge node [pos=0.5, sloped, above] {$c$} (q18)
    (q5) edge node [pos=0.2, sloped, above] {$b$} (q13);

    \path[gray, densely dashed,->] 
    (q3) edge node [pos=0.2, sloped, above] {$a$} (q16)
    (q4) edge node [pos=0.2, sloped, above] {$a$} (q16);

   \node[state] (q1) [above left=\vs cm and 6cm of q0, fill=lightgray] {}; 
   \node[state] (q2) [right=\hsa cm of q1, fill=lightgray] {}; 
   \node[state] (q3) [right=\hsa cm of q2, fill=lightgray] {}; 
   \node[state] (q4) [right=\hsa cm of q3, fill=lightgray] {}; 
   \node[state] (q5) [right=\hsa cm of q4, fill=lightgray] {}; 
   \node[state] (q6) [right=\hsa cm of q5, fill=lightgray] {}; 
   \node[state] (q7) [right=\hsa cm of q6, fill=lightgray] {}; 
   \node[state] (plus) [fill=pink, above=\vS cm of q3] {}; 
   \node[state] (minus) [fill=cyan!20, above=\vS cm of q5] {}; 
   
    \path[->] 
    (q0) edge node [pos=0.4, sloped, left, inner sep=1.2ex] {$f$} (q1)
    (q0) edge node [pos=0.5, sloped, above, inner sep=0.2ex] {$g$} (q2)
    (q0) edge node [pos=0.5, sloped, above, inner sep=0.2ex] {$h$} (q3)
    (q0) edge node [pos=0.5, sloped, right, inner sep=0.2ex, rotate=90] {$i$} (q4)
    (q0) edge node [pos=0.5, sloped, above, inner sep=0.2ex] {$j$} (q5)
    (q0) edge node [pos=0.5, sloped, above, inner sep=0.2ex] {$k$} (q6)
    (q0) edge node [pos=0.4, sloped, right, inner sep=1.2ex] {$\ell$} (q7)

    (q1) edge node [pos=0.3, sloped, above, inner sep=0.2ex] {$a,b,c,d,e$} (plus)
    (q2) edge node [pos=0.3, sloped, above, inner sep=0.2ex] {$d,e,f$} (plus)
    (q2) edge node [pos=0.1, sloped, above, inner sep=0.2ex] {$a$} (minus)
    (q3) edge node [pos=0.5, sloped, above, inner sep=0.2ex] {$f,g,h$} (plus)
    (q3) edge node [pos=0.1, sloped, above, inner sep=0.2ex] {$a,d$} (minus)
    (q4) edge node [pos=0.4, sloped, above, inner sep=0.2ex] {$a,b,c,g,h$} (plus)
    (q4) edge node [pos=0.2, sloped, above, inner sep=0.2ex] {$d$} (minus)
    (q5) edge node [pos=0.1, sloped, above, inner sep=0.2ex] {$d,e$} (plus)
    (q5) edge node [pos=0.6, sloped, above, inner sep=0.2ex] {$b,f,g$} (minus)
    (q6) edge node [pos=0.35, sloped, above, inner sep=0.2ex] {$b,c,e$} (minus)
    (q6) edge node [pos=0.35, sloped, below, inner sep=0.2ex] {$f,g,h$} (minus)
    (q7) edge node [pos=0.1, sloped, below, inner sep=0.0ex] {$f$} (plus)
    (q7) edge node [pos=0.3, sloped, above, inner sep=0.2ex] {$a,c,e,h$} (minus);
    \end{tikzpicture}
}
}
\end{minipage}
\hfill
\begin{minipage}{0.35\textwidth}
\caption{A filter $\struct{F}_{\rm det}$ 
that is tracing-deterministic and vertex single-output. It 
output simulates $\struct{F}_{\rm inp}$  and is a 
solution to the \fm problem in the sense that no other
tracing-deterministic filter with fewer vertices
output simulates $\struct{F}_{\rm inp}$.
The reduction is from $|V(\struct{F}_{\rm inp})|= 36$ to $|V(\struct{F}_{\rm det})|=23$.
\label{fig:fdetout}}
\end{minipage}
\vspace*{-24pt}
\end{figure}

\begin{restatable}[]{lemma}{detmin}
The $\struct{F}_{\rm det}$ of Figure~\ref{fig:fdetout} is a minimal tracing-deterministic filter that output
simulates~$\struct{F}_{\rm inp}$.
\label{lm:minimal_f_det}
\end{restatable}
\begin{proof}
In $\struct{F}_{\rm det}$, other than gray and royal-blue, every other color is
reached by some string that reaches no other color, so at least one vertex must
be present representing that color.  For gray, any pair of the \num{7} vertices
has a pink\slash\hspace{0pt}light-blue conflict on an extension. So none of
those \num{7} pairs can be merged. For the \num{5} royal-blue vertices,
extensions under $a$, $b$, $c$ force each pair apart.
\end{proof}

Now examine Figure~\ref{fig:fndetout}, which gives $\struct{F}_{\rm nd}$, a filter that output
simulates $\struct{F}_{\rm inp}$ and is smaller than $\struct{F}_{\rm
det}$. It is a tracing-\underline{non}deterministic vertex single-output filter. It was
constructed as follows. In the upper-half, pink and light-blue extensions of strings with final symbol $\{a,b,\dots,h\}$ have been
separated and re-constituted in only \num{6} gray vertices.
The lower-half has exploited the fact that some strings have multiple valid outputs (like `$ac$' giving teal or purple); 
using this freedom allows sharing of some vertices.
And this $\struct{F}_{\rm nd}$ is the minimal tracing-nondeterministic filter
that output simulates $\struct{F}_{\rm inp}$. 

\begin{figure}[t]
\vspace*{-1pt}
\begin{minipage}{0.38\textwidth}
\vspace*{3ex}
\caption{A filter $\struct{F}_{\rm nd}$ 
that is tracing-nondeterministic and vertex single-output, and which 
output simulates $\struct{F}_{\rm inp}$. 
It solves Problem~\fm in the sense that no other
filter with fewer vertices output simulates $\struct{F}_{\rm inp}$.
$|V(\struct{F}_{\rm nd})|=21$.
\label{fig:fndetout}}
\end{minipage}
\hfill
\begin{minipage}{0.55\textwidth}
\vspace*{-4ex}
\hspace*{-6ex}
\scalebox{0.5}{
\newcommand\vs{1.5}
\newcommand\vS{2.1}
\newcommand\hs{1.0}
\newcommand\hsb{1.0}
\newcommand\hsa{2.5}
\newcommand\hsbf{3.5}

{
\begin{tikzpicture}[shorten >=1pt,node distance=1cm, on grid,auto]
\tikzset{every state/.style={semithick, minimum size=10pt, inner sep=8pt,circle}}
\tikzset{every path/.style={thick, font=\Large}}
\tikzset{initial distance=5.5cm}
\node[state, initial, fill=white] (q0)   {}; 
   \node[state] (p1) [above left=\vs cm and 6.5cm of q0, fill=lightgray] {}; 
   \node[state] (p2) [right=\hsa of p1, fill=lightgray] {}; 
   \node[state] (p3) [right=\hsa of p2, fill=lightgray] {}; 
   \node[state] (p4) [right=\hsa of p3, fill=lightgray] {}; 
   \node[state] (p5) [right=\hsa of p4, fill=lightgray] {}; 
   \node[state] (p6) [right=\hsa of p5, fill=lightgray] {}; 
   \node[state] (plus) [fill=pink, above left=\vS cm and 0.5cm of p3] {}; 
   \node[state] (minus) [fill=cyan!20, above left=\vS cm and 0.5cm of p5] {}; 
   
    \path[->] 
    (q0) edge node [pos=0.8, sloped, below, inner sep=0.2ex] {$f,i$} (p1)
    (q0) edge node [pos=0.7, sloped, above, inner sep=0.2ex] {$f,g,j$} (p2)
    (q0) edge node [pos=0.5, sloped, above, inner sep=0.2ex] {$g,h,\ell$} (p3)
    (q0) edge node [pos=0.6, sloped, below, inner sep=0.3ex] {$h,i$} (p4)
    (q0) edge node [pos=0.7, sloped, above, inner sep=0.2ex] {$j,k$} (p5)
    (q0) edge node [pos=0.8, sloped, below, inner sep=0.2ex] {$k,\ell$} (p6)
    (p1) edge node [pos=0.5, sloped, above, inner sep=0.2ex] {$a,b,c$} (plus) 
    (p2) edge node [pos=0.5, sloped, above, inner sep=0.2ex] {$d,e$} (plus)
    (p3) edge node [pos=0.5, sloped, left, rotate=90, inner sep=0.2ex] {$f$} (plus)
    (p3) edge node [pos=0.5, sloped, above, inner sep=0.2ex] {$a$} (minus)
    (p4) edge node [pos=0.6, sloped, above, inner sep=0.2ex] {$g,h$} (plus)
    (p4) edge node [pos=0.5, sloped, above, inner sep=0.2ex] {$d$} (minus)
    (p5) edge node [pos=0.45, sloped, above, inner sep=0.2ex] {$b,f,g$} (minus)
    (p6) edge node [pos=0.5, sloped, above, inner sep=0.2ex] {$c,e,h$} (minus);

   \node[state] (p1) [below left=\vs cm and 5cm of q0, fill=blue!35] {}; 
   \node[state] (p2) [right=\hsbf cm of p1, fill=blue!35] {}; 
   \node[state] (p3) [right=\hsbf cm of p2, fill=blue!35] {}; 
   \node[state] (p4) [right=\hsbf cm of p3, fill=blue!35] {}; 
   \node[state] (p5) [below left=\vs and \hsb of p1, fill=green!50] {}; 
   \node[state] (p6) [below right=\vs and \hsb of p1, fill=teal!50] {}; 
   \node[state] (p7) [below left=\vs and \hsb of p2,  fill=red!50] {}; 
   \node[state] (p8) [below right=\vs and \hsb of p2, fill=violet!50] {}; 
   \node[state] (p9) [below left=\vs and \hsb of p3,  fill=brown!90] {}; 
   \node[state] (p10) [below right=\vs and \hsb of p3,fill=lime!20] {}; 
   \node[state] (p11) [below left=\vs and \hsb of p4, fill=yellow!75] {}; 
   \node[state] (p12) [below right=\vs and \hsb of p4,fill=olive!40] {}; 
   
    \path[->] 
    (q0) edge node [pos=0.7, sloped, above, inner sep=0.2ex] {$a,e$} (p1)
    (q0) edge node [pos=0.5, sloped, below, inner sep=0.2ex] {$a,b,c$} (p2)
    (q0) edge node [pos=0.7, sloped, above, inner sep=0.2ex] {$b,d$} (p3)
    (q0) edge node [pos=0.7, sloped, above, inner sep=0.2ex] {$c,d,e$} (p4)
    (p1) edge node [pos=0.7, above, inner sep=1.8ex] {$a$} (p5)
    (p1) edge node [pos=0.7, above, inner sep=1.8ex] {$c$} (p6)
    (p2) edge node [pos=0.7, above, inner sep=1.8ex] {$b$} (p7)
    (p2) edge node [pos=0.7, above, inner sep=1.8ex] {$c$} (p8)
    (p3) edge node [pos=0.7, above, inner sep=1.8ex] {$a$} (p9)
    (p3) edge node [pos=0.7, above, inner sep=1.8ex] {$c$} (p10)
    (p4) edge node [pos=0.7, above, inner sep=1.8ex] {$b$} (p11)
    (p4) edge node [pos=0.7, above, inner sep=1.8ex] {$a$} (p12);

\end{tikzpicture}
}
}
\end{minipage}
\vspace*{-9ex}
\end{figure}
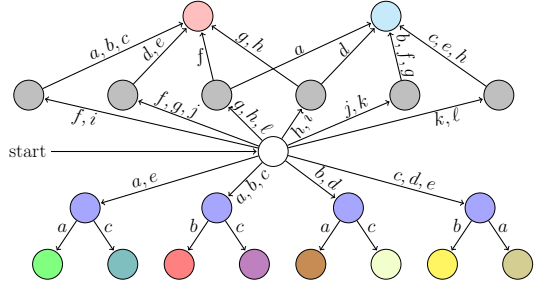

\begin{figure}[b]
\begin{subfigure}[b]{0.19\textwidth}
\includegraphics[width=\textwidth]{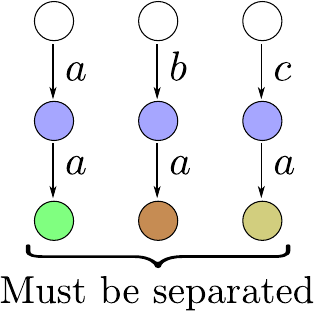}
\caption{Strings~with suffix `$a$'.\label{fig:nd_minimal_reason_a}}
\end{subfigure}
\hfill
\begin{subfigure}[b]{0.19\textwidth}
\includegraphics[width=\textwidth]{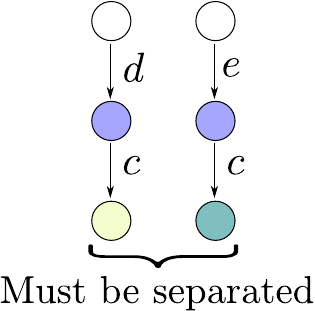}
\caption{Strings~with suffix `$b$'.\label{fig:nd_minimal_reason_b}}
\end{subfigure}
\hfill
\begin{subfigure}[b]{0.54\textwidth}
\includegraphics[width=\textwidth]{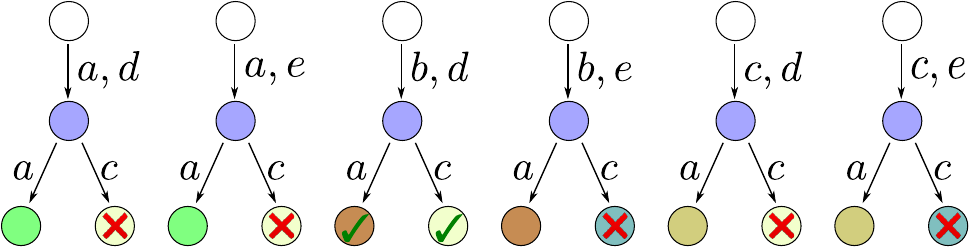}
\quad
\caption{The \num{6} possible ways to overlay strings with suffix `$b$' on the strings ending
with `$a$'.\label{fig:nd_minimal_reason_c}}
\end{subfigure}
\caption{(a). Strings `$aa$', `$ba$', `$ca$' need to go through at least \num{3} different royal-blue states.
(b). Strings `$dc$', `$ec$' need at least \num{2}. (c) When overlaying `$dc$', `$ec$' with
strings `$aa$', `$ba$', `$ca$', `$ec$' cannot go through the same royal-blue state as
any of `$aa$', `$ba$', `$ca$' without causing a conflict. Hence, at least \num{4}
royal-blue states are required to carry strings `$aa$', `$ba$', `$ca$', `$ec$'.
\label{fig:nd_minimal_reason}}
\end{figure}

\begin{restatable}[]{lemma}{ndmin}
$\struct{F}_{\rm nd}$ is a minimal tracing-nondeterministic filter that output
simulates $\struct{F}_{\rm inp}$. 
\label{lm:minimal_f_nd}
\end{restatable}
\begin{proof}
The same argument as before justifies all vertices  that are the sole
representative of their color (i.e., those with colors other than royal-blue
and gray).  Next, we show that $\struct{F}_{\rm nd}$ has the minimal royal-blue
and gray states. As shown in Figure~\ref{fig:nd_minimal_reason_a}, strings
`$aa$', `$ba$', `$ca$' must go through at least \num{3} different royal-blue vertices.
Otherwise, if two of them go through the same royal-blue, then it will produce
two different outputs for each of those two strings, which violates output
simulation.  Similarly in Figure~\ref{fig:nd_minimal_reason_b}, `$dc$' and
`$ec$' must go though at least \num{2} different royal-blue states. Without any
limit, we might create \num{5} different royal-blue for the above five
strings to avoid conflicts. But to use only \num{3} or fewer royal-blue states,
strings `$dc$',`$ec$' have to be
overlaid with `$aa$', `$ba$', `$ca$' such that `$dc$', `$ec$' go
through the same royal-blue states as `$aa$', `$ba$', `$ca$' do. There are \num{6} ways
as shown in Figure~\ref{fig:nd_minimal_reason_c}: only `$dc$' can be overlaid
with `$ba$'.  The others will cause a conflict. For example, if `$ec$' visits the
same royal-blue state as `$aa$' does, then `$ac$' outputs lime-green, which is
incompatible with the original output, teal for `$ac$', in $\struct{F}_{\rm
inp}$. Therefore, we need at least \num{4} royal-blue states in the
tracing-nondeterministic minimizer: a total of \num{3} for `$aa$', `$ba$', `$ca$', and \num{1} for
`$ec$'. Using the same argument, we need at least \num{6} gray states.
Therefore, $\struct{F}_{\rm nd}$ has the minimal number of states for every
color, and hence is minimal.  
\end{proof}


\newcommand{\planfig}{%
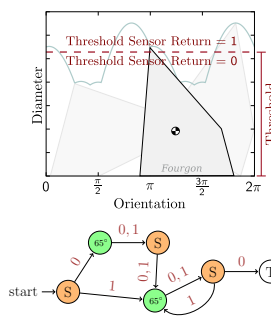
\begin{wrapfigure}[23]{r}{0.3\textwidth}
\centering
\vspace*{-4ex}
\includegraphics[width=0.28\textwidth]{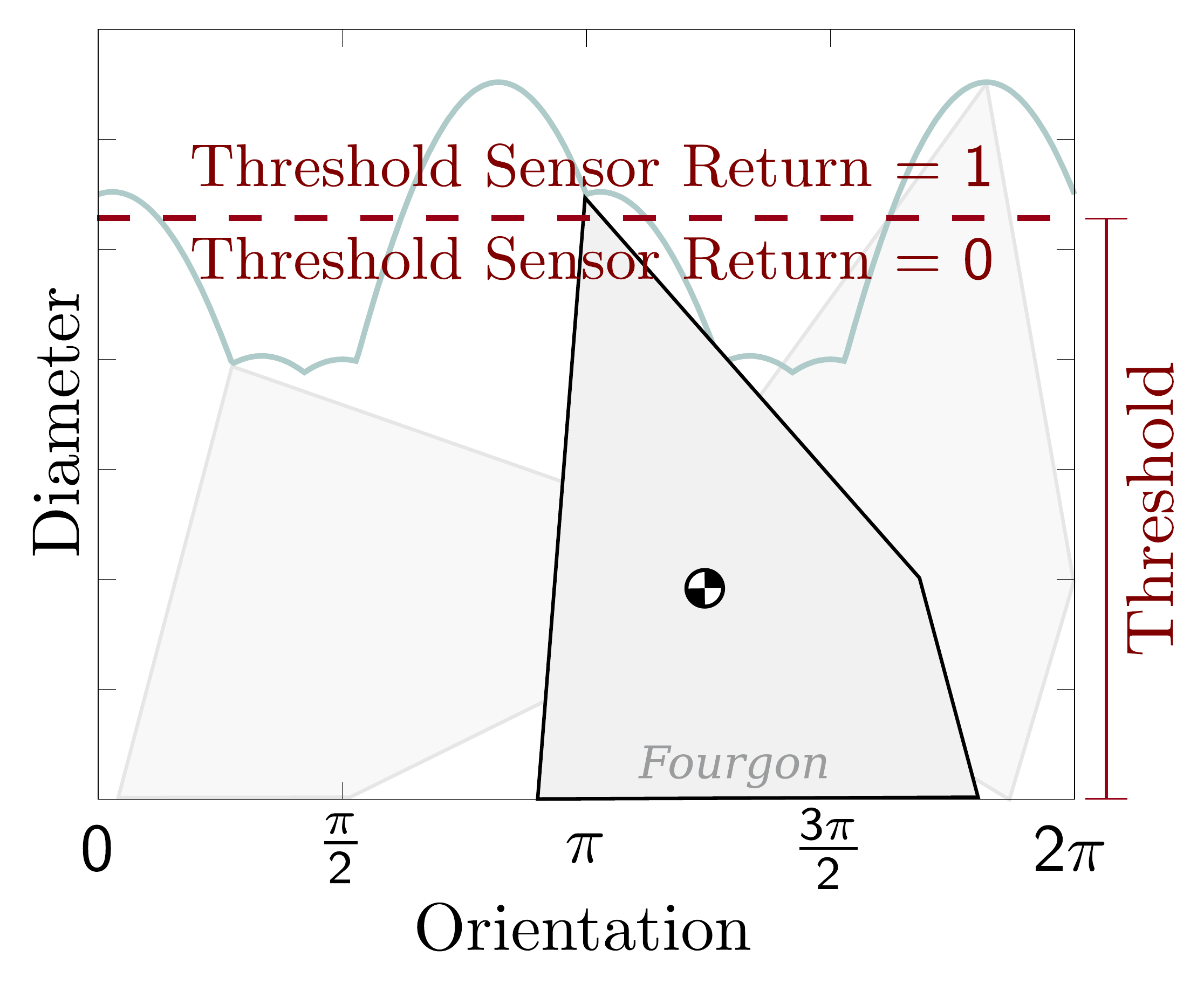}
\vspace*{-8ex}
\hspace*{-2ex}
\scalebox{0.385}{
{
\begin{tikzpicture}[shorten >=1pt,node distance=1cm, on grid,auto]
\tikzset{every state/.style={semithick, minimum size=10pt, inner sep=8pt,circle}}
\tikzset{every path/.style={thick, font=\Large}}
\tikzset{initial distance=0.5cm}
\node[state, initial, fill=orange!55] (s)   {}; 
   \node[state] (p1) [above right=1.7 cm and 1.1cm of s, fill=green!45] {}; 
   \node[state] (s1) [right=2.0 cm of p1, fill=orange!55] {}; 
   \node[state] (p2) [below right=0.3 cm and 3.0cm of s, fill=green!45] {}; 
   \node[state] (s2) [above right=1.0 cm and 2.0 cm of p2, fill=orange!55] {}; 
   \node[state] (t) [above right=0.00cm and 2.0 cm of s2, fill=white] {}; 
   \node[align=center] at (t) {T};
   \node[align=center] at (s) {S};
   \node[align=center] at (s1) {S};
   \node[align=center] at (s2) {S};
   \node[align=left] at (p1) {\small$\phantom{l}{65^\circ}$};
   \node[align=left] at (p2) {\small$\phantom{l}{65^\circ}$};
   
    \path[->] 
    (s) edge node [pos=0.5, sloped, above, inner sep=1.2ex, color=black!45!red!80] {$0$} (p1)
    (s) edge node [pos=0.5, sloped, above, inner sep=1.2ex, color=black!45!red!80] {$1$} (p2)
    (p1) edge node [pos=0.5, sloped, above, inner sep=1.2ex, color=black!45!red!80] {$0,1$} (s1)
    (p2) edge node [pos=0.5, sloped, above, inner sep=1.2ex, color=black!45!red!80] {$0,1$} (s2)
    (s2) edge node [pos=0.5, sloped, above, inner sep=1.2ex, color=black!45!red!80] {$0$} (t)
    (s1) edge node [pos=0.5, sloped, above, inner sep=1.2ex, color=black!45!red!80] {$0,1$} (p2)
    ;

    \path[->] 
    (s2) edge[bend angle=-70,bend right] node [pos=0.5, sloped, above, inner sep=1.2ex, color=black!45!red!80] {$1$} (p2);


\end{tikzpicture}
}
}
\vspace*{3ex}
\caption{A small feedback plan for a planar manipulation problem expessesd as a filter.  
Using \textcolor{black!45!red!80}{$0$} or \textcolor{black!45!red!80}{$1$}
readings from a threshold sensor, it orients the \num{4}-sided part shown via a
sequence of \textcolor{black!25!orange!95}{squeeze-gripper} and
\textcolor{black!25!green!95}{rotate-by-$65^\circ$} actions. (Based on
\cite[Figs.~14 and 15]{o2017concise}.) 
\label{fig:planoout}}

\end{wrapfigure}}

\subsection{Processing inputs incrementally and string single-output filters}
Thus far, filters have been discussed almost as though they were entirely
abstract objects.  When employed in practice, their output (an element from set
$C$, which we've shown visually through colors) is grounded to some specific
meaning.  It may have an interpretation as a sufficient statistic for some
estimation task, or as an answer to a particular query, or it may be an action
to have the robot execute.  Figure~\ref{fig:planoout} provides an example of a
feedback plan for a manipulation 
\begin{wrapfigure}[23]{r}{0.3\textwidth}
\centering
\vspace*{-4ex}
\includegraphics[width=0.28\textwidth]{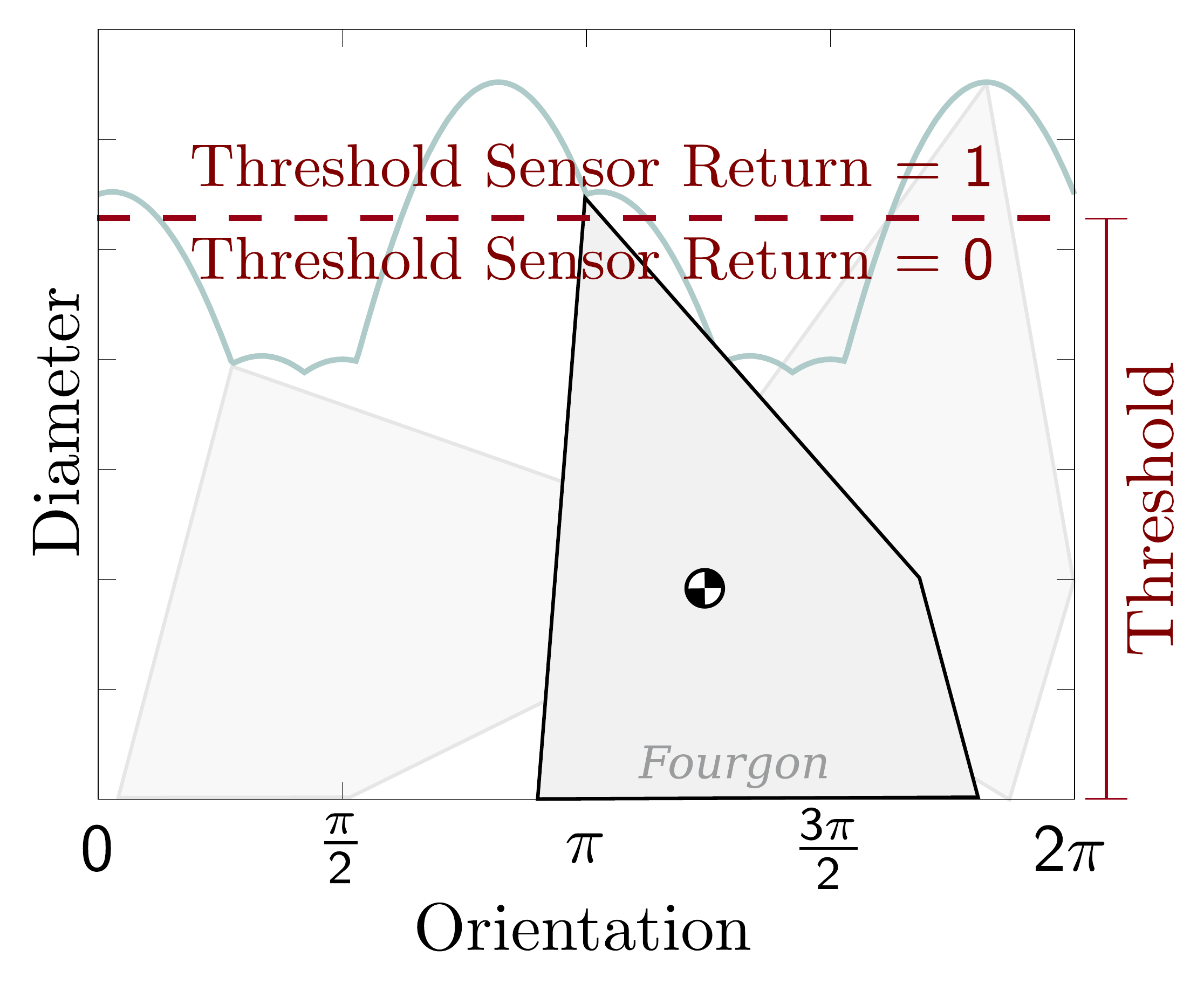}
\vspace*{-8ex}
\hspace*{-2ex}
\scalebox{0.385}{
{
\begin{tikzpicture}[shorten >=1pt,node distance=1cm, on grid,auto]
\tikzset{every state/.style={semithick, minimum size=10pt, inner sep=8pt,circle}}
\tikzset{every path/.style={thick, font=\Large}}
\tikzset{initial distance=0.5cm}
\node[state, initial, fill=orange!55] (s)   {}; 
   \node[state] (p1) [above right=1.7 cm and 1.1cm of s, fill=green!45] {}; 
   \node[state] (s1) [right=2.0 cm of p1, fill=orange!55] {}; 
   \node[state] (p2) [below right=0.3 cm and 3.0cm of s, fill=green!45] {}; 
   \node[state] (s2) [above right=1.0 cm and 2.0 cm of p2, fill=orange!55] {}; 
   \node[state] (t) [above right=0.00cm and 2.0 cm of s2, fill=white] {}; 
   \node[align=center] at (t) {T};
   \node[align=center] at (s) {S};
   \node[align=center] at (s1) {S};
   \node[align=center] at (s2) {S};
   \node[align=left] at (p1) {\small$\phantom{l}{65^\circ}$};
   \node[align=left] at (p2) {\small$\phantom{l}{65^\circ}$};
   
    \path[->] 
    (s) edge node [pos=0.5, sloped, above, inner sep=1.2ex, color=black!45!red!80] {$0$} (p1)
    (s) edge node [pos=0.5, sloped, above, inner sep=1.2ex, color=black!45!red!80] {$1$} (p2)
    (p1) edge node [pos=0.5, sloped, above, inner sep=1.2ex, color=black!45!red!80] {$0,1$} (s1)
    (p2) edge node [pos=0.5, sloped, above, inner sep=1.2ex, color=black!45!red!80] {$0,1$} (s2)
    (s2) edge node [pos=0.5, sloped, above, inner sep=1.2ex, color=black!45!red!80] {$0$} (t)
    (s1) edge node [pos=0.5, sloped, above, inner sep=1.2ex, color=black!45!red!80] {$0,1$} (p2)
    ;

    \path[->] 
    (s2) edge[bend angle=-70,bend right] node [pos=0.5, sloped, above, inner sep=1.2ex, color=black!45!red!80] {$1$} (p2);


\end{tikzpicture}
}
}
\vspace*{3ex}
\caption{A small feedback plan for a planar manipulation problem expessesd as a filter.  
Using \textcolor{black!45!red!80}{$0$} or \textcolor{black!45!red!80}{$1$}
readings from a threshold sensor, it orients the \num{4}-sided part shown via a
sequence of \textcolor{black!25!orange!95}{squeeze-gripper} and
\textcolor{black!25!green!95}{rotate-by-$65^\circ$} actions. (Based on
\cite[Figs.~14 and 15]{o2017concise}.) 
\label{fig:planoout}}

\end{wrapfigure}
problem (re-drawn from~\cite{o2017concise},
itself inspired from~\cite{taylor88sensor} and work in this line). The outputs
describe how the system should operate: orange and green map to actions to
execute (specific motions for a gripper), and white indicates that the system
should terminate as the task has been completed successfully.

We envision a controller for the robot being represented internally via a
encoding of this filter, which is then used to process sensor
observations and generate actions. It does this incrementally.  For instance, in Figure~\ref{fig:planoout}, 
the controller simply has to track the current vertex to know what to have the
robot do.  
Minimizing the filter has the advantage, then, of reducing the controller's
memory footprint.  So now consider what happens when, in light the previous
examples of superior reduction via nondeterminism, the minimized filter
exhibits some nondeterminism.  Suppose something simple like the \num{6}-state
filter in Figure~\ref{fig:ex_sso_first} is obtained as a result.

For the filter in Figure~\ref{fig:ex_sso_first}, when the robot obtains,
first, an `$a$', it has a choice from the initial white vertex. It might make the
arbitrary decision to take the lower-branch. 
If the sensor provides it an `$x$' next, it would proceed to the orange state.
Then, if a `$y$' is obtained (possible, since `$axy$' is in the input language
and so some output should be produced) it has then become clear that the initial
choice could've been better made if the upper-branch were taken.  (And
introducing the dashed $z$-edge and yellow-state shows that, though in the basic
setup the upper-branch has a superset of the strings of the lower-branch, there
needn't be a choice the covers the other options.) In order to resolve this
missing $y$, the filter effectively needs to rewind back to the `$a$' and take a
different edge. Such an approach is undesirable because it requires the
controller store the (unbounded) history of readings, and the filter is no
longer operating in an efficient incremental manner.

One solution is to trace forward on all matching edges: so in
Figure~\ref{fig:ex_sso_first}, after `$a$', some markers keep the position on
both upper- and lower-branches. Then, when a `$y$' arrives, though one trace
dies off, the second can proceed. This requires a number of markers, but they are bounded
in the size of the filter. One question, when one has multiple vertices that
are currently reached (so both blue and red after `$a$' in Figure~\ref{fig:ex_sso_first}),
is which output should be chosen? This motivates the following definition, wherein we
are spared the burden of making a choice:

\begin{definition}[\sso]
\label{def:sso}
A filter $\struct{F}=(V, V_0, Y, \tau, \allowbreak C, c)$ is \defemp{\sso} if
$\forall s\in \Language{\struct{F}}$, $|\reachedc{\struct{F}}{s}|=1$; Otherwise,
call it \defemp{\smo}. 
\end{definition}

\textsl{Immediate remark:} Any tracing-deterministic vertex single-output
filter will be a \sso filter.

In the discussion earlier, we made the case that the rewinding interpretation
for tracing-nondeterminism is undesirable from a practical point of view.  But
notice how Definition~\ref{def:sso} comes into play also under that interpretation: for a
general tracing-nondeterministic filter, when a string is rewound and traced
forward anew, it might generate a different sequence of output colors.
Returning to our concrete example in Figure~\ref{fig:ex_sso_first}, following
the symbol `$y$' being encountered on the lower-branch, after
rewinding and tracing `$a$' forward a second time on the upper-branch now,  
the color blue is associated with the string.  If this filter is encoding a plan,
then
blue corresponds to an action; but it is not appropriate to generate that
action\,---that represents a point in the past---\,as the action for red was
already executed, with sensor reading `$x$' being produced thereafter.
Had both branches produced the 
same outputs (that is, been \sso), then even if the robot decided to store the
sequence of input symbols and rewind to process the string forward, 
avoiding crashing, and reaching a vertex from which to resume tracing at the newest
symbol, this would be a purely internal affair. 
The actions generated in the world 
would be consistent with the ones actually executed earlier in time.

Is the class of \sso filters useful? We can answer in the affirmative. Suppose
one wishes to minimize a filter, but maintain the sort of temporal/causal consistency
just described. Then one may seek a \sso minimizer and, indeed, there exist
such filters that are smaller than any deterministic filter.  Returning to
example filter $\struct{F}_{\rm inp}$ of Figure~\ref{fig:finp}, now examine
$\struct{F}_{\rm sso}$ presented in Figure~\ref{fig:fsooout}.
It is tracing-nondeterministic but is \sso, and it has one fewer vertex than 
$\struct{F}_{\rm det}$.
Further, this particular filter $\struct{F}_{\rm sso}$ is a \sso minimizer
of $\struct{F}_{\rm inp}$. 

\begin{restatable}[]{lemma}{ssomin}
$\struct{F}_{\rm sso}$ is a \sso minimizer of $\struct{F}_{\rm inp}$.
\label{lm:minimal_f_sso}
\end{restatable}
\begin{proof}
The upper-half follows the same argument as that from
Lemma~\ref{lm:minimal_f_nd}, and the upper-half of $\struct{F}_{\rm sso}$ is 
minimal. For the lower-half, when overlaying `$ec$' with `$dc$', string `$bc$'
outputs both purple and lime-green, which is not \sso. Hence, neither of
`$dc$',`$ec$' can be overlayed with strings `$aa$', $`ba'$, $ca$', so as to create a \sso
filter. Hence, we need at least \num{5} royal-blue states: \num{3} for
`$aa$', `$ba$', `$ca$', and \num{2} more for `$dc$', `$ec$', and the lower-half of $\struct{F}_{\rm
sso}$ is also minimal. 
\end{proof}

\begin{figure}[t]
\vspace*{-1ex}
\begin{minipage}{0.54\textwidth}
\vspace*{-5ex}
\hspace*{-3ex}
\scalebox{0.5}{
\newcommand\vs{1.5}
\newcommand\vS{2.1}
\newcommand\hs{1.0}
\newcommand\hsb{1.0}
\newcommand\hsa{2.5}
{
\begin{tikzpicture}[shorten >=1pt,node distance=1cm, on grid,auto]
\tikzset{every state/.style={semithick, minimum size=10pt, inner sep=8pt,circle}}
\tikzset{every path/.style={thick, font=\Large}}
\tikzset{initial distance=5.5cm}
\node[state, initial, fill=white] (q0)   {}; 
   \node[state] (p1) [above left=\vs cm and 6.5cm of q0, fill=lightgray] {}; 
   \node[state] (p2) [right=\hsa of p1, fill=lightgray] {}; 
   \node[state] (p3) [right=\hsa of p2, fill=lightgray] {}; 
   \node[state] (p4) [right=\hsa of p3, fill=lightgray] {}; 
   \node[state] (p5) [right=\hsa of p4, fill=lightgray] {}; 
   \node[state] (p6) [right=\hsa of p5, fill=lightgray] {}; 
   \node[state] (plus) [fill=pink, above left=\vS cm and 0.5cm of p3] {}; 
   \node[state] (minus) [fill=cyan!20, above left=\vS cm and 0.5cm of p5] {}; 
   
    \path[->] 
    (q0) edge node [pos=0.8, sloped, below, inner sep=0.2ex] {$f,i$} (p1)
    (q0) edge node [pos=0.7, sloped, above, inner sep=0.2ex] {$f,g,j$} (p2)
    (q0) edge node [pos=0.5, sloped, above, inner sep=0.2ex] {$g,h,\ell$} (p3)
    (q0) edge node [pos=0.6, sloped, below, inner sep=0.3ex] {$h,i$} (p4)
    (q0) edge node [pos=0.7, sloped, above, inner sep=0.2ex] {$j,k$} (p5)
    (q0) edge node [pos=0.8, sloped, below, inner sep=0.2ex] {$k,\ell$} (p6)
    (p1) edge node [pos=0.5, sloped, above, inner sep=0.2ex] {$a,b,c$} (plus) 
    (p2) edge node [pos=0.5, sloped, above, inner sep=0.2ex] {$d,e$} (plus)
    (p3) edge node [pos=0.5, sloped, left, rotate=90, inner sep=0.2ex] {$f$} (plus)
    (p3) edge node [pos=0.5, sloped, above, inner sep=0.2ex] {$a$} (minus)
    (p4) edge node [pos=0.6, sloped, above, inner sep=0.2ex] {$g,h$} (plus)
    (p4) edge node [pos=0.5, sloped, above, inner sep=0.2ex] {$d$} (minus)
    (p5) edge node [pos=0.45, sloped, above, inner sep=0.2ex] {$b,f,g$} (minus)
    (p6) edge node [pos=0.5, sloped, above, inner sep=0.2ex] {$c,e,h$} (minus);

   \node[state] (q1) [below left=\vs cm and 6cm of q0, fill=blue!35] {}; 
   \node[state] (q2) [right=3cm of q1,              fill=blue!35] {}; 
   \node[state] (q5) [right=3cm of q2,              fill=blue!35] {}; 
   \node[state] (q3) [right=3cm of q5,              fill=blue!35] {}; 
   \node[state] (q4) [right=3cm of q3,              fill=blue!35] {}; 
   \node[state] (q6) [below left=\vs cm and \hs cm of q1, fill=green!50] {}; 
   \node[state] (q8) [below right=\vs cm and \hs cm of q1, fill=red!50] {}; 
   \node[state] (q9) [below =\vs cm of q2, fill=brown!90] {}; 
   \node[state] (q12) [below =\vs cm of q3, fill=lime!20] {}; 
   \node[state] (q13) [below left=\vs cm and \hs cm of q4, fill=yellow!75] {}; 
   \node[state] (q15) [below right=\vs cm and \hs cm of q4, fill=teal!50] {}; 
   \node[state] (q18) [below left=\vs cm and \hs cm of q5, fill=violet!50] {}; 
   \node[state] (q16) [below right=\vs cm and \hs cm of q5, fill=olive!40] {}; 
   
    \path[->] 
    (q0) edge node [pos=0.7, sloped, above] {$a$} (q1)
    (q0) edge node [pos=0.7, sloped, above] {$b$} (q2)
    (q0) edge node [pos=0.7, sloped, above] {$d$} (q3)
    (q0) edge node [pos=0.7, sloped, above] {$e$} (q4)
    (q0) edge node [pos=0.4, sloped, right, rotate=90] {$c$} (q5)

    (q1) edge node [pos=0.5, sloped, above] {$a$} (q6)
    (q1) edge node [pos=0.4, sloped, above] {$b$} (q8)
    (q1) edge node [pos=0.2, sloped, above] {$c$} (q18)
    (q2) edge node [pos=0.3, sloped, above] {$b$} (q8)
    (q2) edge node [pos=0.3, sloped, right, rotate=90] {$a$} (q9)
    (q2) edge node [pos=0.3, sloped, above] {$c$} (q18)
    (q3) edge node [pos=0.3, sloped, right, rotate=90] {$c$} (q12)
    (q3) edge node [pos=0.2, sloped, above] {$b$} (q13)
    (q4) edge node [pos=0.5, sloped, above] {$b$} (q13)
    (q4) edge node [pos=0.5, sloped, above] {$c$} (q15)
    (q5) edge node [pos=0.5, sloped, above] {$a$} (q16)
    (q5) edge node [pos=0.5, sloped, above] {$c$} (q18)
    (q5) edge node [pos=0.2, sloped, above] {$b$} (q13);
        

    \path[gray, densely dashed,->] 
    (q3) edge node [pos=0.2, sloped, above] {$a$} (q16)
    (q4) edge node [pos=0.2, sloped, above] {$a$} (q16);

\end{tikzpicture}
}
}
\end{minipage}
\hfill
\begin{minipage}{0.36\textwidth}
\caption{A filter $\struct{F}_{\rm sso}$ 
that is tracing-nondeterministic and vertex single-output, and which 
output simulates $\struct{F}_{\rm inp}$. 
It solves \fm in the sense that no other
\sso
filter with fewer vertices output simulates $\struct{F}_{\rm inp}$.
$|V(\struct{F}_{\rm sso})|=22$.
\label{fig:fsooout}}
\end{minipage}
\vspace*{-8ex}
\end{figure}


One additional point worth noting is that $\struct{F}_{\rm nd}$ has a
language which is larger than that of $\struct{F}_{\rm inp}$. Both
$\struct{F}_{\rm det}$ and $\struct{F}_{\rm sso}$ 
match the input language exactly.  Specifically, $\struct{F}_{\rm nd}$ will
process strings `$da$' and `$ea$', and assign them outputs. These show up as a
sort of aliasing in the compression down to \num{4} royal-blue vertices.  It would be
erroneous to believe that this is necessary for the gap in minimizer sizes.  If
one adds the dashed gray $a$-edges to Figures~\ref{fig:finp},
\ref{fig:fdetout}, and ~\ref{fig:fsooout}, then all have precisely the same
language.

\vspace*{-0.9ex}
\subsection{Classes of minimizer} 
\vspace*{-1pt}

The three filters $\struct{F}_{\rm det}$, $\struct{F}_{\rm sso}$, and
$\struct{F}_{\rm nd}$ are all minimizers of $\struct{F}_{\rm inp}$ within their
respective classes.  The class of tracing-non\-deter\-ministic filters, by
definition, omits the constraint that forces deter\-minism, and thus contains all the
tracing-deter\-ministic ones as well.  As already remarked, the class of \sso
filters includes all the tracing-deter\-ministic ones.  And
tracing-non\-deter\-ministic filters may well violate the \sso constraint.  So, as
solutions to the problem of minimization, tracing-deter\-ministic ones can be
no smaller than tracing-non\-deter\-ministic or \sso ones; and \sso ones may be
no smaller than tracing-non\-deter\-ministic ones. Of course, for a particular
problem, there \textsl{might} be a size gap the other way, but there needn't
always be.  
(Consider, for instance, Figure~\ref{fig:ex_sso_first}; it is its
own  nondeterministic minimizer, but is also the deterministic minimizer as
well.) 
The three $\struct{F}_{\rm det}$, $\struct{F}_{\rm sso}$, and
$\struct{F}_{\rm nd}$ show a separation in their sizes for the single input of
$\struct{F}_{\rm inp}$.

This motivating discussion now concluded, we will always be explicit in what
follows about the class of minimizer sought when studying a variant of \fm. 
Further, we will also clarify in assumptions on the input filter.
We shall abbreviate:
\vspace*{-8pt}
\label{sec:summary_of_opts}
\begin{tightenumerate}
\item [\dfm:] where both the input and output filter are
tracing-deterministic ({\sc df}), e.g. $\struct{F}_{\rm det}$.
\item [\sfm:] where the input filter is tracing-deterministic ({\sc df}), and the
output filter must be \sso ({\sc sso}), e.g. $\struct{F}_{\rm sso}$ or  $\struct{F}_{\rm sso}$ plus dashed $a$-edges.
\item [\mfm:] where the input filter is tracing-deterministic ({\sc df}), and the output
filter can be tracing-nondeterministic ({\sc smo}), e.g. $\struct{F}_{\rm nd}$.
\end{tightenumerate}
\vspace*{-8pt}

\subsection{Aside: single- and multi-output vertices} 

One observation is that $\struct{F}_{\rm det}$ and $\struct{F}_{\rm nd}$ are vertex single-output filters. 
Indeed, no minimizer need have ever multi-output vertices itself.
\begin{lemma}
If $\struct{F}_{m}$ is a vertex multi-output filter that output simulates
$\struct{F}$, then there exists some $\struct{F}_{s}$
as an output simulator of $\struct{F}$, where 
$\struct{F}_{s}$ is vertex single-output 
and, moreover, $|V(\struct{F}_{m})|=
|V(\struct{F}_{s})|$.
\end{lemma}
\begin{proof}
Any vertex in $\struct{F}_{m}$ with multiple outputs can have a single one selected arbitrarily to yield
$\struct{F}_{s}$. Then $\struct{F}_{s}$ output simulates $\struct{F}_{m}$, and hence $\struct{F}$ too.
\end{proof}
\vspace*{-1ex}

Multi-output vertices allow expression of some flexibility in the $Y$ to
$C$ mapping.  As already apparent in Figure~\ref{fig:finp} in
the lead-up to Figures~\ref{fig:fdetout}, \ref{fig:fndetout}, and~\ref{fig:fsooout}, 
multi-output vertices can be helpful for expressing some input
to a minimization problem, specifically in providing a specification which
grants a degree of freedom.
(Both \cite{zhang20cover} and \cite{zhang19accelerating}, in
their respective introductory sections, have detailed robotics scenarios in which there is a
degree of freedom which is natural, and is cleanly expressed via multi-output
vertices.)

The preceding helps emphasize that one use of filters is as specifications of
the appropriate range of system responses for a given input.  This is a
distinct use from their adoption directly as the object that computes outputs,
incrementally, from inputs as and when they arrive.

\vspace*{-3pt}
\subsection{The role of output simulation and consistency across sequences}
\vspace*{-3pt}

Let's revisit the topics arising in the discussion directly following
Definition~\ref{def:sso}, but now from a point of view that puts aside
considerations of specific edge/vertex structure.  Any filter $\struct{F}$ that
is \sso, when interpreted as an input--output map on sequences, i.e., from
$\Language{\struct{F}} \subseteq \KleeneStr{Y}$ to $\KleeneStr{C}$, is
deterministic.  That is, it is a function.  Though the filter may be
tracing-nondeterministic, any of those branching choices are purely internal to
it, and any observer treating it as a black-box is shielded from those
decisions.

For some given filter $\struct{G}$, its \sso minimizer must produce, for each
string, some output compatible with that produced on $\struct{G}$.
Language inclusion
ensures that the minimizer can process every string that filter $\struct{G}$ can, and all
filters process a prefix-closed collection of strings.  But it is important to
note that Definition~\ref{def:stdos} does not place any additional requirement
on the relationship between the outputs produced for a string and those
produced by its prefixes.  Or, when thinking of the filter as a function on
sequences: if sequence~$s$ maps to~$t$, then a subsequence of~$s$ need not map
to a subsequence of~$t$.  Put differently: it is a mapping from
$\Language{\struct{F}} \subseteq \KleeneStr{Y}$ to $C$ not to $\KleeneStr{C}$,
and when one naturally uses the prefix closure property to interpret it as lifted to
sequences of $C$, something extra appears, seemingly.
This \textsl{extra} is not in the requirements for output simulation.

Specifically, no `tracing correlations' need be preserved. To make this
concrete, see in Figure~\ref{fig:ex_sso_input} how on string `$ax$' only blue
followed by green, or red followed by orange, can be produced. Output
simulation permits any filter that produces either red or blue on `$a$', and
green or orange on `$ax$'.  This freedom (four cases rather than only two) may be necessary for behavior to be
implementable incrementally.  Using our example: after adding in the $z$-edge/yellow-vertex,
Figure~\ref{fig:ex_sso_second} gives a tracing-deterministic output simulator
that does not preserve the red--orange--yellow output ordering; but no
tracing-deterministic output simulator exists that can preserve the output
orderings produced by both `$axy$' and `$axz$' strings.

In the very next section, this will not be a concern for a
much simpler reason:  when we consider the hardness of forms of minimization to
\sso filters compared to other cases (i.e., the \sfm problem \emph{vs.} \dfm and
\mfm), we shall assume we are given a tracing-deterministic input.
(Without this restriction, minimization can only become harder---a fact we use in Section~\ref{sec:diff})  And the process
of converting some general filter into a tracing-deter\-min\-istic input (e.g.
\cite[Algorithm~2]{setlabelrss}) removes the tracing dependent output
orderings.


\section{Hardness of {\titledfm}, \titlesfm, and \titlemfm}
\label{sec:generalcase}
Having seen the gap between the minimizers of \dfm, \sfm
and \mfm, we now examine the computational complexity of
minimization.

\subsection{Complexity of \titledfm, \titlesfm, and \titlemfm}
In this section, we will give worst-case complexity analyses for the \num{3} problems.


Prior work has proved the hardness results for \dfm by reducing from a graph
coloring problem:
\begin{theorem}[Theorem 9~\cite{zhang20cover}]
\dfm is \npcomplete.
\end{theorem}

Next, we will leverage the existing results to show that both \mfm and \sfm are
\pspacecomplete. 


Prior results in tracing-nondeterministic filter minimization show that 
to check whether one tracing-nondeterministic filter output simulates another
is in \pspace.

\begin{lemma}[Lemma~7~\cite{zhang2021nondeterminism}]
\label{lm:nd_fm_space}
Given a tracing-nondeterministic filter $\struct{F}$ and a
tracing-nondeterministic filter $\struct{F'}$, it is in \pspace to check whether
$\struct{F'}$ output simulates $\struct{F}$. 
\end{lemma} 

In both \mfm and \sfm, the input filter is tracing-deterministic, which is a
special case in Lemma~\ref{lm:nd_fm_space}. Hence, it is also \pspace to check
output simulation for \mfm and \sfm. Accordingly, we have the \pspace results for \mfm
immediately:
\begin{lemma}[\titlemfm is in \pspace]
Given a tracing-deterministic filter $\struct{F}$ and a tracing-nondeterministic
filter $\struct{F'}$, it is in \pspace to check whether $\struct{F'}$ output
simulates $\struct{F}$. 
\end{lemma}
But for \sfm, we additionally need to check whether the output filter is \sso,
and are required to show whether this procedure is in \pspace as well. 

To facilitate the proof for \sfm, we use the
following graph product:
\begin{definition}[product graph]
Given filters $\struct{F}$ and $\struct{F'}$, all the strings that are in both  
$\struct{F}$ and $\struct{F'}$ are produced by their tensor product graph
$\struct{G}$, denoted as $\struct{G}=\struct{F}\pgproduct\struct{F'}$.
$\struct{G}$ has initial states $V_0(\struct{F})\times V_0(\struct{F'})$, and  
for every string $s\in \Language{\struct{F}}\cap\Language{\struct{F'}}$,
$(v, v')\in \reachedv{\struct{G}}{s}$. 
\end{definition}

Using the above operator, we show that \sfm is also in \pspace:
\begin{lemma}[\titlesfm is in \pspace]
Given a determinsitic filter $\struct{F}$ and an non-determinsitic filter
$\struct{F'}$, it is in \pspace to check whether $\struct{F'}$ is \sso, and 
$\struct{F'}$ output simulates $\struct{F}$ or not
is in \pspace.
\label{lm:smfm_pspace}
\end{lemma}
\begin{proof}
Following Lemma~\ref{lm:nd_fm_space}, it is in \pspace to check whether
the output simulating condition holds for \sfm. Next, we will show that it also
takes polynomial space to check whether $\struct{F'}$ is \sso or not. First,
if there is a state $v'$ in $\struct{F'}$ that has more than a single output, i.e.,
$|c(v')|>1$, then the strings in $\reaching{\struct{F'}}{v'}$ have more than one
output, and $\struct{F'}$ is not \sso. Otherwise, we build a product with itself
$\struct{G}=\struct{F'}\pgproduct \struct{F'}$ to vet states
that are reached by the same string. According to the
construction, there are at most $|V(\struct{F'})|^2$ states in $\struct{G}$.
Additionally, every pair of states that are nondeterministically reached by the
same string in $\struct{F'}$ will appear as a state in $\struct{G}$. And every
state in $\struct{G}$ consists of two states that are 
reached by some string in $\struct{F'}$. Hence, the states in $\struct{G}$
capture all pairs of states that are non-deterministically reached. For every
string $s$ in $\struct{F'}$, let $V'_s=\reachedv{\struct{F'}}{s}$ be the set of
states non-deterministically reached by $s$. Then we know that $s$ has precisely one
output if and only if every pair of states in $V'_s$ share the same
output mutually. Therefore, we can say $\struct{F'}$ is \sso, if for every state
$(v'_i, v'_j)$ in $\struct{G}$, $c(v'_i)=c(v'_j)$. Otherwise, it is not. 
So we have a polynomial space procedure.
\end{proof}

Next, we will examine the hardness results for both \sfm and \mfm via
some other results from automata theory. Similar to a filter, an automaton $\struct{A}$
is defined as a tuple $(V_0, V, \Sigma, \tau, F)$, where $\Sigma$ is the
alphabet, $F$ is the set of final states. Different from the filter,
the language of an automaton is called the accepting language, denoted henceforth as
$\ALanguage{\struct{A}}$, which is the set of strings that reach the final states.
But the difference between $\ALanguage{\cdot}$ and just $\Language{\cdot}$
disappears when all states in the automaton are final.
Automata in which $F = V$ are
called {\sc ASF} automata (where `{\sc ASF}' stands for `All States Final').    

A prior result for {\sc ASF} automata shows that it is \pspacecomplete to check 
whether an {\sc ASF} NFA with alphabet $\Sigma$ has accepting language
$\KleeneStr{\Sigma}$ or not: 
\begin{lemma}[{\sc NFA-NonUniversality-ASF}~\cite{kao2009nfas}] 
Given an {\sc ASF} NFA $\struct{A}$ with alphabet $\Sigma$, if $|\Sigma|\geq 2$,
it is \pspacecomplete to check whether $\ALanguage{\struct{A}} = \KleeneStr{\Sigma}$
holds or not. 
\end{lemma}

Then we will show that both \sfm and \mfm are \pspacehard by reducing from
{\sc NFA-NonUniversality-ASF}.
\begin{lemma}[\titlesfm and \titlemfm are \pspacehard]
\label{lm:smfm_pspacehard}
Given a tracing-deterministic filter $\struct{F}$ and an \sso
tracing-nondeterministic filter $\struct{F'}$, if $|Y(\struct{F})|\geq 2$ and
$|Y(\struct{F'})|\geq 2$, it is \pspacehard to check whether $\struct{F'}$
output simulates $\struct{F}$ or not. 
\end{lemma}
\begin{proof}
Proof by reduction from {\sc NFA-NonUniversality-ASF}. Given an {\sc ASF} NFA
$\struct{A}$ with alphabet $\Sigma$ and $|\Sigma|\geq 2$, then treat
$\struct{A}$ as a filter $\struct{F}'$ with an output function that colors every state
the same color $c_0$. Then the interaction language of filter $\struct{F}'$ is the same
as the accepting language of automata $\struct{A}$, i.e.,
$\Language{\struct{F'}}=\ALanguage{\struct{A}}$.  Next, create a
tracing-deterministic filter $\struct{F}$, where there is only a single state
with a self-loop bearing $\Sigma$. This state is colored~$c_0$.
Then $\Language{\struct{F}}=\KleeneStr{\Sigma}$. Therefore,
$\ALanguage{\struct{A}}=\KleeneStr{\Sigma}\Longleftrightarrow
\Language{\struct{F}}\subseteq\Language{\struct{F'}}$. Hence, to check {\sc
NFA-NonUniversality-ASF}, i.e., whether $\ALanguage{\struct{A}}=\KleeneStr{\Sigma}$
holds or not, is equivalent to checking 
$\Language{\struct{F}}\subseteq\Language{\struct{F'}}$.
Additionally, for every string $s$ in $\struct{F}$, the output from $\struct{F}$
is the same as that from $\struct{F'}$. Hence,
$\Language{\struct{F}}\subseteq\Language{\struct{F'}}$ if and only if
$\struct{F'}$ output simulates $\struct{F}$. Therefore, {\sc
NFA-NonUniversality-ASF} has been reduced to \sfm in polynomial time, and \sfm is 
\pspacehard. The same reduction also shows that \mfm is \pspacehard.   
\end{proof}
Note that this reduction requires the input to have a non-unitary 
alphabet, so as to be general enough to model the {\sc NFA-NonUniversality-ASF}
problem.   

\begin{theorem}
Both \sfm and \mfm are \pspacecomplete.
\end{theorem}
\begin{proof}
Combine Lemma~\ref{lm:smfm_pspace} and Lemma~\ref{lm:smfm_pspacehard}.
\end{proof}

\subsection{Minimization problems with a unitary alphabet}
Lemma~\ref{lm:smfm_pspacehard} indicates that when the alphabet 
comprises \num{2} or more symbols, both \sfm and \mfm are
\pspacehard. We examine hardness results for unitary alphabet
versions of the problems.

With a unitary alphabet, the tracing-deterministic filter has either a finite chain of
states, or a finite chain with a cycle attached at the end of the chain. 
In such cases, which we write as \unitdfm can be solved efficiently: 
\begin{theorem}[\titleunitdfm is in \p]
Given a tracing-deterministic input filter $\struct{F}$ with $|Y(\struct{F})|=1$
(unitary alphabet $Y=\{y\}$), then finding the minimal tracing-deterministic
filter that output simulates $\struct{F}$ is in \p.
\label{lm:unitary_deterministic}
\end{theorem}
\begin{proof}
We prove the hardness by giving a polynomial algorithm. 
First, there is always a tracing-deterministic minimizer that also has a unitary alphabet. 
Given some minimizer that is otherwise, simply remove the
labels that are not in $Y$ and the edges accordingly.
Since the minimizer is tracing-deterministic with a unitary alphabet, it can either be
(1) a finite chain of states (for finite language), or (2) a finite chain with a
cycle (for $\KleeneStr{Y}$). For any type-(1) minimizer, we can add a self-loop
bearing label $y$ at the leaf node, obtaining a type-(2) minimizer, being no larger
while also output simulating the input filter.
Thus, the tracing-deterministic minimizer can be parameterized as a type-(2) filter,
i.e.,a finite chain with $k$ states and a cycle with $m$ states, where $k\in
\Naturals$ and $m\in\PositiveNaturals$. Since the $\struct{F}$ output simulates itself, the minimizer
can be no larger. Let
$n=|V(\struct{F})|$. Then, with $k\in \{0,1,\dots, n\}$ and $m\in
\{1,2,\dots, n\}$ we can enumerate the $O(n^2)$ potential tracing-deterministic
minimizers candidates. For each filter $\struct{F}^{\dagger}$, it takes polynomial time to check
whether $\struct{F}^{\dagger}$ output simulates $\struct{F}$ or not: First,
denote the initial state from  $\struct{F}$ and $\struct{F}^{\dagger}$ as $v_0$
and $v^{\dagger}_0$ respectively. Then check whether $c(v_0)\supseteq
c(v_0^{\dagger})$ or not. If not
output simulation of $\struct{F}^{\dagger}$ is violated.  Otherwise, move on
to check their $y$-children. If the state from $\struct{F}$ has a $y$-child but
the state from $\struct{F}^{\dagger}$ does not, then violation (owing to failure of language inclusion)
has been detected.
Once all $\struct{F}$'s $y$-children have been checked, output simulation is satisfied.
It is, thus, in \p to find the tracing-deterministic minimizer for $\struct{F}$.
\end{proof}

Next, we will show that the tracing-nondeterministic minimizer for a tracing-deterministic
input filter with unitary alphabet is tracing-deterministic, and it can be found in
polynomial time as well: 

\begin{restatable}[\titleunitsfm and \titleunitmfm are in \p]{theorem}{unitnfm}
Given a tracing-deterministic input filter $\struct{F}$ with $|Y(\struct{F})|=1$
(unitary alphabet $Y=\{y\}$), then it is \p to find the minimal 
tracing-nondeterministic filter that output simulates $\struct{F}$.  
\label{lm:unitary_nd}
\end{restatable}
\begin{sproof}
One shows (see Appendix for details) that this reduces to the case in
Theorem~\ref{lm:unitary_deterministic}.
\end{sproof}

So far, no difference has manifested in the hardness of finding \sso \emph{vs.}
general tracing-nondeterministic minimizers.  The general cases are both
intractable, and the special unitary cases, both efficient.
The next section tries to probe this difference.

\section{Differences between \sso minimization and general
\label{sec:diff}
tracing-nondeterministic minimization}

In order to better understand whether there is any complexity difference
between \sso and general tracing-nondeterministic minimizers, we next consider
input filters that may also include tracing-nondeterministic ones.  The
previous results apply to both equally because they leverage the first
requirement of output simulation, namely language inclusion.  The second
requirement, output compatibility, relates to the colors generated, so seems
more likely to be where a difference, if any exists, might be pinpointed.
Thus, for this more general class of inputs, we lay open the monolithic
definition of output simulating, examining the component strands, i.e.,
both requirements separately. 
As we show, language inclusions for both \sso and general
tracing-nondeterministic minimization is \pspacehard, but they differ
in the difficulty of checking output compatibility. 
Output compatibility for general tracing-nondeterministic minimizers is 
\pspacehard to check, but it only takes polynomial time to check output
compatibility for \sso tracing-nondeterministic ones.



The naming convention used in the problems before is now extended: 
\gsfmOC indicates both the input and output filter are \sso and potentially
tracing-nondeterministic ({\sc sso}), while
\gmfmOC means both input and output filters are
tracing-nondeterministic ({\sc smo}).

\smallskip
Following the results from the previous section, it is \pspacehard to check
language inclusion for \gsfm or \gmfm: 
\begin{lemma}
Given a \sso tracing-nondeterministic filter $\struct{F}$ and $\struct{F'}$, it
is \pspacehard to check whether $\Language{\struct{F}}\subseteq
\Language{\struct{F'}}$ holds or not.  \label{lm:nd_language_inclusion}
\end{lemma}
\begin{proof}
This has been proved by the reduction presented as Lemma~\ref{lm:smfm_pspacehard}. 
\end{proof}


A difference does show up in checking outputs:
the \sso property means 
output compatibility for \gsfmOC can be checked efficiently:

\begin{lemma}[output compatibility for \titlegsfmOC is in \p]
\label{lem:checkingsso}
Given a tracing-nondeterministic filter $\struct{F}$ and a non-deterministic filter
$\struct{F'}$ such that $\Language{\struct{F}}\subseteq \Language{\struct{F'}}$,
then checking output compatibility: $\forall s\in
\Language{\struct{F}}, \reachedc{\struct{F}}{s}\supseteq
\reachedc{\struct{F'}}{s}$ is in \p.
\label{lm:sso_output_compatibility}
\end{lemma}
\begin{proof}
We will give a polynomial time procedure to check whether this property holds or
not. First, construct a tensor product graph
$\struct{G}=\struct{F}\pgproduct\struct{F'}$. Next, for every state $v'$ in
$\struct{F'}$, collect the set of states in $\struct{F}$ that are paired with it
in $\struct{G}$, and call it $R_{v'}=\lbrace v\in
V(\struct{F})| (v, v')\in V(\struct{G})\rbrace$. $R_{v'}$ can be constructed in
polynomial time. Then $R_{v'}$ contains the set of all states from $\struct{F}$
that are reached by some string that reaches $v'$ in $\struct{F'}$. If there is
no string from $\struct{F}$ that reaches $v'$ in $\struct{F'}$, i.e.,
$\reaching{\struct{F'}}{v'}\cap \Language{\struct{F}}=\emptyset$, then
$R_{v'}=\emptyset$. Since both $\struct{F}$ and $\struct{F}'$ are \sso, they are
also \vso. For each $v'$, if the output of $v'$ is the same as that of every
state in $R_{v'}$, then it satisfies the output simulation criterion. Otherwise, if there
exists a state $v\in R_{v'}$ such that $c(v)\neq c(v')$, then it violates output
simulation, since $\reaching{\struct{F}}{v}\cap \reaching{\struct{F}'}{v'}$ have output
that is incompatible. Therefore, we only need to check whether the output of
every state paired with $v'$ in $\struct{G}$ is the same as that of $v'$, which
can be done in polynomial time.
\end{proof}

However, without input filter being \sso, checking output compatibility is
\pspacehard:
\begin{lemma}[output compatibility for \titlegmfmOC is \pspacehard]
\label{lem:checkingsmo}
Given a tracing-non\-deter\-ministic input filter $\struct{F}$ and a
tracing-non\-deter\-ministic filter $\struct{F'}$ such that
$\Language{\struct{F}}\subseteq \Language{\struct{F'}}$, then it is \pspacehard
to check the output compatibility:
$\forall s\in \Language{\struct{F}}, \reachedc{\struct{F}}{s}\supseteq
\reachedc{\struct{F'}}{s}$. 
\end{lemma}
\begin{proof}
Via reduction from language inclusion, which is shown to be \pspacehard in
Lemma~\ref{lm:nd_language_inclusion}. Given non-deterministic filters $\struct{A}$
and $\struct{B}$ such that $\Language{\struct{A}}\subseteq
\Language{\struct{B}}$, to check the property of output compatibility, we first
construct a product graph between $\struct{A}$ and $\struct{B}$, denoted as
$\struct{J}=\struct{A}\pgproduct\struct{B}$. Then we construct a graph union
with $\struct{J}$ and $\struct{A}$, calling it $\struct{G}$, consisting of
vertices and edges from both $\struct{J}$ and $\struct{A}$. Hence, we have
$\Language{\struct{G}}= \Language{\struct{A}}$. Next, treat $\struct{G}$
as a filter $\struct{F}$ by coloring the vertices from $\struct{J}$ green,
and vertices from $\struct{A}$ red. Treat $\struct{A}$ as a filter
$\struct{F'}$ with all states green. Then, we have
$\Language{\struct{F'}}=\Language{\struct{A}}=\Language{\struct{G}}=\Language{\struct{F}}$.
If $\Language{\struct{A}}\subseteq \Language{\struct{B}}$, then for every string
$s\in\Language{\struct{A}}$, $s$ will reach a red state (from $\struct{A}$) and
a green state (from $\struct{J}$) in $\struct{F}$. Hence
$\reachedc{F}{s}=\lbrace {\rm red}, {\rm green}\rbrace$. Since $\forall s\in
\Language{F},$ $\reachedc{F'}{s}=\lbrace {\rm green}\rbrace$, we have that
$\struct{F}$ and $\struct{F'}$ are output compatible. If
$\Language{\struct{A}}\not\subseteq \Language{\struct{B}}$, then there exists
a string $s$ such that $s\in
\Language{\struct{A}}$ but $s\not \in \Language{\struct{B}}$. Hence, $s$ will
only reach red states (from $\struct{A}$) in $\struct{F}$. Since $s$ reaches 
green states in $\struct{F'}$, then it violates the output compatibility property.
Therefore, the problem of checking language inclusion was reduced  in
polynomial time to the problem of checking output compatibility. So the problem
is \pspacehard as well. 
\end{proof}


\begin{figure}[h]
\vspace*{-1pt}
\begin{minipage}{0.7\textwidth}
\centering
\vspace*{-4ex}
\hspace*{-1ex}
\includegraphics[width=\textwidth]{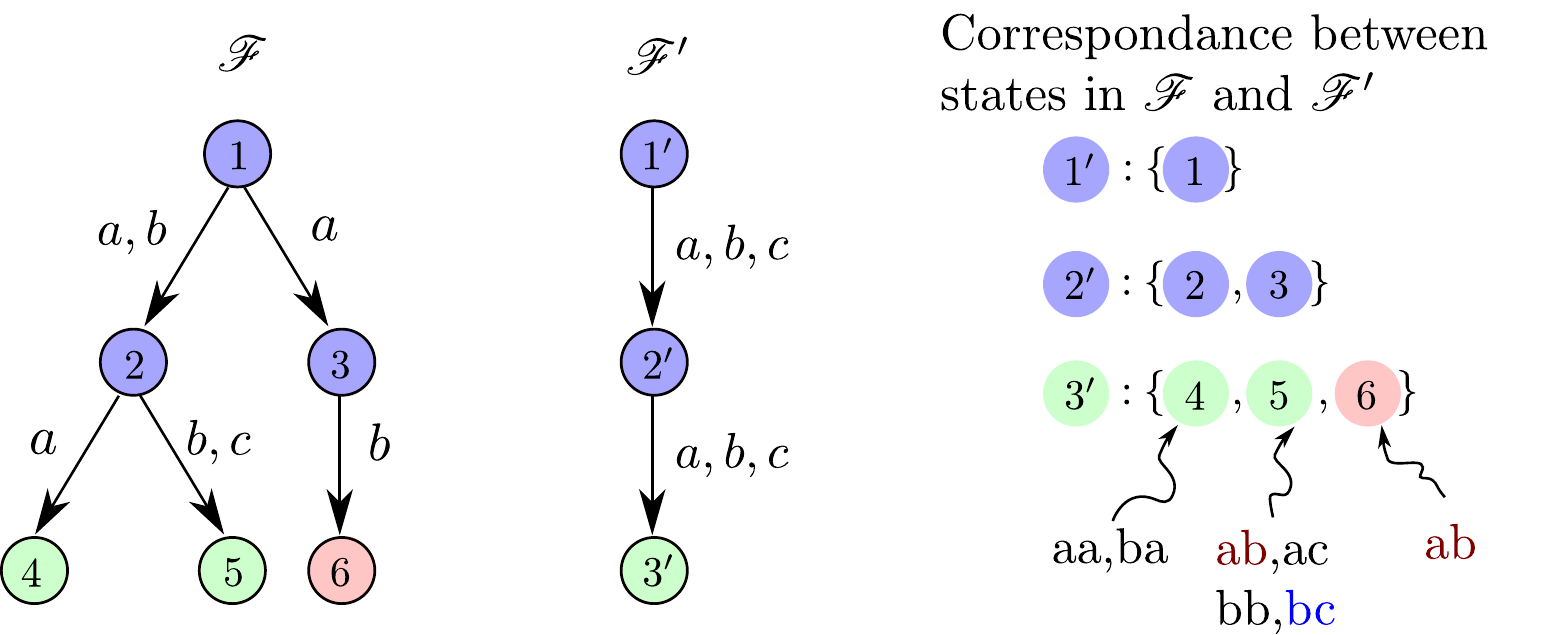}
\end{minipage}
\hfill
\begin{minipage}{0.275\textwidth}
\caption{An illustration that develops intuition for the hardness of checking output compatibility for \gmfmOC.\label{fig:output_compatibility_hard}}
\end{minipage}
\vspace*{-4ex}
\end{figure}

The difference in hardness for checking output compatibility between \gsfmOC and \gmfmOC
can be seen intuitively. As illustrated via Figure~\ref{fig:output_compatibility_hard}: a
tracing-nondeterministic input filter $\struct{F}$ is shown at the left, and an
output filter $\struct{F'}$ is shown in the middle. For every state $v'$ in
$\struct{F'}$, we have collected the set of states $R_{v'}$ in $\struct{F}$ that are
reached by the strings reaching $v'$, exactly as in the procedure in
Lemma~\ref{lm:sso_output_compatibility}. For every $v'$, its corresponding set
$R_{v'}$ is shown at the right of
Figure~\ref{fig:output_compatibility_hard}. As illustrated in
Lemma~\ref{lm:sso_output_compatibility}, if $\struct{F}$ is \sso, then we only
need to check the output of every individual state inside each $R_{v'}$ for
$v'$. We assert a violation of output compatibility if there is a state whose
output is different from that of $v'$. But if $\struct{F}$ is not \sso, we
cannot immediately claim that a failure was found when an output inconsistency is found between
state $v\in R_{v'}$ and $v'$ because the strings $\reaching{\struct{F}}{v}\cap
\reaching{\struct{F}'}{v'}$ may non-deterministically reach other states that
share the same output as $v'$. An example is shown in
Figure~\ref{fig:output_compatibility_hard}, for vertex $3'$, $R_{3'}=\{4, 5,
6\}$. Vertex $6$ outputs pink, while $3'$ outputs green. So there is an
inconsistency. But the strings `$ab$' reaching $6$ and $3'$, also reach a 
green state nondeterministically. In this case,
$\reachedc{\struct{F'}}{ab}\subseteq \reachedc{\struct{F}}{ab}$ still holds.
Therefore, when $\struct{F}$ is not \sso, one must check the output across all
the states that are nondeterministically reached in $\struct{F}$, which is
\pspacehard.

Speaking informally, given that checking language inclusion is at least as hard
as either of the output compatibility checks, one expects to pay at least that
price for determining output simulation. Lemma~\ref{lem:checkingsso} suggests
that price must be paid twice for \gmfmOC, in the sense that there are two
instances of this problem embedded in corroborating output simulation.  While
for \gsfmOC the language inclusion check seems to dominate.  We note that the
conditions on the last two lemmas may have some relevance in applications, for
instance, when one has domain knowledge (or an oracle) that tells you language
inclusion holds.

\vspace*{-8pt}
\section{Summary and Conclusion}
\vspace*{-1ex}

This paper explores a new type of nondeterminism that
is novel in the context of combinatorial filtering.  
It has
attractive properties when used, for instance, in encoding feedback
plans/policies concisely: First, the outputs that such filters produce are consistent,
isolating specific tracing choices (or rewinding operations) made during
processing from being manifested externally. As mappings, they exhibit
deterministic behavior. Secondly, they provide degrees of freedom 
absent from deterministic filters, which facilitate greater compression. It is
curious that this should actually be possible, but our example
demonstrates a clear separation in sizes between the classes.
To initiate study of this class, the \sso filters, we have examined
hardness of size minimization, establishing that the general problems are of
the same complexity class as classical nondeterminism. 
\Sso filter minimization is \pspacehard in terms of the language inclusion for
output simulation, and not output compatibility; general nondeterministic
filter minimization, it turns out, is \pspacehard in terms of both the properties\,---\,a fact never before realized.


\vspace*{-1ex}

%
%
%
{\small
\bibliographystyle{splncs04}
\bibliography{mybib}
}
\end{document}